\PassOptionsToPackage{numbers,sort&compress}{natbib}
\documentclass{article}

\usepackage[preprint]{neurips_2026}

\usepackage[utf8]{inputenc} 
\usepackage[T1]{fontenc} 
\usepackage[colorlinks=true,citecolor=blue,linkcolor=blue,urlcolor=blue]{hyperref}     
\usepackage{url}            
\usepackage{booktabs}       
\usepackage{amsfonts}       
\usepackage{nicefrac}       
\usepackage{microtype}      
\usepackage{xcolor}         
\usepackage{rotating}
\usepackage{amsmath,amssymb}
\usepackage{amsthm}
\usepackage{tikz-cd}
\usetikzlibrary{positioning}
\usepackage{multicol}
\newtheorem{lemma}{Lemma}
\usepackage{algorithm}
\usepackage{algpseudocode}
\usepackage{multirow}
\newtheorem*{theorem*}{Theorem}
\usepackage{caption}
\usepackage{placeins}

\usepackage{fontawesome5}
\usepackage{caption}
\usepackage{amsthm}

\theoremstyle{plain}   
\newtheorem{theorem}{Theorem}[section]

\newtheorem{proposition}{Proposition}[section]
\newtheorem{corollary}{Corollary}

\theoremstyle{definition}  

\title{Local Inverse Geometry Can Be Amortized}

\author{
Aaditya L. Kachhadiya\thanks{\texttt{kachhadiyaaaditya@gmail.com
}}\\
Independent Researcher\\
Surat, India\\
\href{https://github.com/AadityaKachhadiya/deceptron}{\textcolor{blue}{\faGithub\ github.com/AadityaKachhadiya/deceptron}}
}

\begin{document}

\maketitle

\begin{abstract}
Nonlinear inverse problems often trade inexpensive but fragile first-order updates against curvature-aware methods such as Gauss--Newton and Levenberg--Marquardt, which obtain stronger directions by repeatedly solving Jacobian-based linearized systems. We propose a learned alternative: amortize local inverse geometry into a reusable reverse operator. Our framework learns a bidirectional surrogate, Deceptron, and deploys it through D-IPG (Deceptron Inverse-Preconditioned Gradient), an iterative solver that pulls residual-corrected measurement-space proposals back to latent space. The key mechanism is a Jacobian Composition Penalty (JCP), which trains the reverse Jacobian to act as a local left inverse of the forward Jacobian; its runtime counterpart, RJCP, measures the same inverse-consistency error along optimization trajectories. We prove that D-IPG is first-order equivalent to damped Gauss--Newton under local pseudoinverse consistency, with deviation controlled by composition error and conditioning. Across seven PDE inverse-problem benchmarks, D-IPG outperforms standard baselines, achieves \(94.8\%\) mean success across the six-problem reliability suite, and reaches comparable or better recovery quality at up to \(77\times\) lower inference-time solve cost on the main benchmarks.

\end{abstract}

\section{Introduction}
\label{sec:introduction}

Nonlinear inverse problems require updates that are simultaneously informative, stable, and computationally tractable. Existing optimizers divide sharply along these axes: first-order methods are inexpensive but can be slow or unreliable when the forward Jacobian is ill-conditioned, whereas curvature-aware methods such as Gauss--Newton and Levenberg--Marquardt provide stronger local directions by repeatedly solving linearized inverse problems \citep{levenberg1944method,marquardt1963algorithm,nocedal2006numerical}. The resulting bottleneck is therefore not only forward evaluation, but repeated reconstruction of local inverse geometry during optimization.

This paper asks whether that local inverse geometry can instead be amortized: learned once and reused across inverse solves. The goal is not to replace the forward model by a one-shot inverse predictor, but to learn a reverse operator whose differential action converts measurement residuals into informative latent-space corrections. Concretely, we introduce a learned bidirectional module, \emph{Deceptron}, consisting of a differentiable forward surrogate and a reverse map trained to provide local inverse-oriented differential structure. At inference time, Deceptron induces \emph{D-IPG} (Deceptron Inverse-Preconditioned Gradient), an iterative solver that forms residual-corrected proposals in measurement space and maps them back to latent space through the learned reverse model.

The central technical requirement is therefore differential rather than functional. For D-IPG to behave like a useful inverse solver, the learned reverse Jacobian must approximate the local inverse action of the forward Jacobian on the directions that matter for nonlinear least squares. We promote this structure through a Jacobian Composition Penalty (JCP), together with a corresponding runtime diagnostic, RJCP, that respectively shape and measure local inverse consistency of the learned reverse Jacobian. Section~\ref{sec:theory} shows that, under local pseudoinverse consistency, the induced D-IPG proposal is first-order equivalent to damped Gauss--Newton, and that its deviation from the damped Gauss--Newton direction is controlled by composition error and local conditioning.

We evaluate the proposed framework on PDE inverse problems spanning 1-, 2-, and 3-dimensional heat-equation initial-condition recovery, together with harder 2-dimensional systems including Darcy flow, advection--diffusion, Allen--Cahn, and Navier--Stokes inversion. Across these settings, D-IPG consistently outperforms standard first-order baselines, remains competitive with strong second-order methods, and often attains similar recovery quality at substantially lower runtime. More broadly, the paper advances a simple viewpoint: local inverse geometry, traditionally reconstructed repeatedly through classical second-order machinery, can be amortized into a learned reverse operator and reused inside a safeguarded iterative solver.

\section{Related Work}
\label{sec:related}

Our work is most closely related to classical nonlinear least-squares methods, especially Gauss--Newton, Levenberg--Marquardt, and quasi-Newton procedures, which exploit local curvature information through Jacobian- or Hessian-based linear solves \citep{levenberg1944method,marquardt1963algorithm,nocedal2006numerical}. These methods remain strong baselines for inverse problems, but their cost is dominated by repeated reconstruction of local inverse structure at each iterate. Our objective is not to discard this geometry, but to amortize part of its local inverse action into a reusable learned operator.

The paper is also related to learned inverse solvers, amortized inference, and unrolled optimization methods, which use learned mappings or fixed-horizon iterative architectures to accelerate inverse recovery \citep{gregor2010learning,andrychowicz2016learning,monga2021algorithm,adler2018learned}. Our setting differs in two respects. First, the learned reverse map is not itself the full solver: it supplies local inverse-oriented update information to a separate inference-time optimizer equipped with explicit descent safeguards. Second, our target is not direct functional inversion but usable local inverse action for iterative nonlinear least squares.

More broadly, operator-learning approaches such as DeepONets and Fourier neural operators learn solution or inverse maps between function spaces \citep{lu2021learning,li2021fourier}, while learned preconditioners aim to accelerate iterative linear solves. Our setting differs from both strands. We do not learn a direct solution operator from observation to latent state, nor a preconditioner for a fixed linear system; instead, we learn a reverse differential operator queried along the trajectory of a nonlinear least-squares solver. Finally, our use of Jacobian composition differs from standard reconstruction or cycle-consistency losses: JCP is introduced not as a representation regularizer, but as an optimization-theoretic constraint that promotes local inverse action of the learned reverse Jacobian and thereby supports a Gauss--Newton-like update mechanism.

\section{Method}
\label{sec:method}

We consider inverse recovery under a differentiable surrogate forward map
\[
f_W:\mathbb{R}^{d_{\mathrm{in}}}\to\mathbb{R}^{d_{\mathrm{out}}},
\]
where \(x\in\mathbb{R}^{d_{\mathrm{in}}}\) denotes the latent variable and \(y^\star\in\mathbb{R}^{d_{\mathrm{out}}}\) the observed measurement. At inference time, all learned model parameters are fixed, and optimization is performed only over \(x\). We work with the surrogate nonlinear least-squares objective
\begin{equation}
\Phi(x):=\frac{1}{2}\|f_W(x)-y^\star\|_2^2,
\qquad
r(x):=f_W(x)-y^\star,
\label{eq:method_objective}
\end{equation}
with Jacobian \(J_f(x):=\nabla_x f_W(x)\). The central objective is to replace repeated Jacobian-based local inverse reconstruction by a learned reverse operator that can be reused across iterations. Throughout, all norms are Euclidean for vectors and induced spectral norms for matrices unless stated otherwise.

\begin{algorithm}[t]
\caption{D-IPG inference-time update}
\label{alg:dipg}
\begin{algorithmic}[1]
\Require Observation \(y^\star\), learned module \((f_W,g_V)\), initial state \(x_0\), feasible set \(\mathcal C\), relaxation \(\rho\in(0,1]\), initial step size \(\alpha_0>0\), Armijo constant \(c\in(0,1)\), backtracking factor \(\beta\in(0,1)\), maximum iterations \(T\)
\For{\(t=0,1,\dots,T-1\)}
    \State \(y_t \gets f_W(x_t)\), \quad \(r_t \gets y_t-y^\star\)
    \If{stopping criterion satisfied}
        \State \textbf{break}
    \EndIf
    \State \(\alpha \gets \alpha_0\)
    \Repeat
        \State \(y^{\mathrm{prop}} \gets y_t-\alpha r_t\)
        \State \(x^{\mathrm{prop}} \gets g_V(y^{\mathrm{prop}})\)
        \State \(\widetilde{x} \gets \Pi_{\mathcal C}\!\big((1-\rho)x_t+\rho x^{\mathrm{prop}}\big)\)
        \State \(p_t \gets \widetilde{x}-x_t\)
        \If{\(\Phi(\widetilde{x}) \le \Phi(x_t)+c\,\nabla\Phi(x_t)^\top p_t\)}
            \State \(x_{t+1}\gets \widetilde{x}\)
            \State \textbf{accept and continue}
        \Else
            \State \(\alpha \gets \beta\alpha\)
        \EndIf
    \Until{step accepted or backtracking budget exhausted}
    \If{no step accepted}
        \State \textbf{break}
    \EndIf
\EndFor
\State \Return \(x_t\)
\end{algorithmic}
\end{algorithm}

\subsection{D-IPG and learned reverse maps}
\label{sec:dipg}

Our framework separates the learned object from the inference-time solver built on top of it. \emph{Deceptron} denotes the learned bidirectional pair \((f_W,g_V)\), where \(f_W\) is the forward surrogate and
\[
g_V:\mathbb{R}^{d_{\mathrm{out}}}\to\mathbb{R}^{d_{\mathrm{in}}}
\]
is a learned reverse map. \emph{D-IPG}, short for Deceptron Inverse-Preconditioned Gradient, is the iterative optimizer induced by this module. The reverse map is not intended to be a globally exact inverse of \(f_W\); rather, it is trained to provide local inverse-oriented update information relative to the geometry induced by \(f_W\). Thus the learned module carries amortized local inverse structure, while the solver remains an explicit optimization procedure with standard acceptance safeguards.

At iteration \(t\), let
\begin{equation}
y_t=f_W(x_t),\qquad r_t=y_t-y^\star.
\label{eq:yt_rt}
\end{equation}
D-IPG first forms a residual-corrected proposal in measurement space,
\begin{equation}
y^{\mathrm{prop}}_{t+1}=y_t-\alpha_t r_t,
\label{eq:y_prop}
\end{equation}
then maps this proposal through the learned reverse model,
\begin{equation}
x^{\mathrm{prop}}_{t+1}=g_V\!\left(y^{\mathrm{prop}}_{t+1}\right),
\label{eq:x_prop}
\end{equation}
and finally applies relaxation and projection,
\begin{equation}
x_{t+1}
=
\Pi_{\mathcal C}\!\left((1-\rho)x_t+\rho x^{\mathrm{prop}}_{t+1}\right),
\label{eq:x_update}
\end{equation}
where \(\Pi_{\mathcal C}\) denotes projection onto a feasible set \(\mathcal C\) and \(\rho\in(0,1]\) is a relaxation parameter. The step size \(\alpha_t\) is selected by Armijo backtracking on \(\Phi\) \citep{armijo1966minimization}, so the method retains an explicit sufficient-decrease acceptance test even though the proposal direction is generated through the learned reverse map.

The unit value \(\alpha_t=1\) has a natural interpretation. Before relaxation and projection, \eqref{eq:y_prop} sends the measurement-space proposal exactly to \(y^\star\), so \(\alpha_t=1\) is the canonical residual-correction scale when \(g_V\) is locally inverse-consistent. Smaller values under-correct in measurement space, while larger values extrapolate past the observation and can enter an incorrect basin.

Although \eqref{eq:y_prop} is written in measurement space, D-IPG is fundamentally an \(x\)-space solver: the optimization variable is always \(x_t\), and \(y^{\mathrm{prop}}_{t+1}\) serves only as a residual-corrected point at which the learned reverse map can be evaluated. The relation to classical second-order methods therefore arises from the induced first-order \(x\)-space behavior of \eqref{eq:x_prop}, not from a separate optimization problem in \(y\). In particular, the proposal should be interpreted as a structured pullback of a residual correction, not as a reparameterization of the inverse problem itself.

Algorithm~\ref{alg:dipg} summarizes the inference-time update. The Jacobian Composition Penalty (JCP), defined in Section~\ref{sec:local_inverse}, is a training-time penalty only; it is not evaluated during D-IPG inference and therefore adds no runtime cost. Because the proposal is generated through an amortized reverse map, Armijo backtracking is used as an acceptance safeguard rather than as a guarantee that sufficiently small \(\alpha\) is always accepted.\footnote{Appendix~\ref{app:armijo_amortized} discusses this distinction.}

The algorithm preserves the outer structure of a safeguarded nonlinear solver: the learned reverse map supplies proposals, while acceptance remains tied to the explicit objective \(\Phi\). Thus learned inverse geometry and descent control remain separated.

\subsection{Jacobian Composition Penalty (JCP)}
\label{sec:local_inverse}

The value of the reverse map \(g_V\) does not lie in exact inversion of the full forward surrogate. In nonlinear inverse problems, global inversion is generally unavailable, and even local one-shot inversion is not the operative requirement for iterative optimization. What the solver needs is a reusable local operator that converts measurement-space residual information into informative latent-space corrections. Thus the relevant object is not only the function value \(g_V(y)\), but its Jacobian \(J_g(y):=\nabla_y g_V(y)\).

The central differential quantity is the Jacobian composition
\begin{equation}
J_g(f_W(x))\,J_f(x).
\label{eq:jac_comp}
\end{equation}
When \eqref{eq:jac_comp} is close to the identity on \(\mathbb{R}^{d_{\mathrm{in}}}\), the learned reverse Jacobian acts as an approximate local left inverse of the forward Jacobian: infinitesimal perturbations in latent space are approximately preserved after forward propagation and reverse mapping. This condition is local and differential; it does not assert that \(g_V\) globally inverts \(f_W\), but rather that its Jacobian carries the first-order inverse action needed for iterative inverse optimization.

We promote this structure through a Jacobian Composition Penalty (JCP). Let \(\xi\in\mathbb{R}^{d_{\mathrm{in}}}\) be a random probe vector with i.i.d.\ Rademacher or standard normal entries, so that \(\mathbb{E}[\xi\xi^\top]=I\). By Hutchinson's estimator \citep{hutchinson1990stochastic},
\[
\mathbb{E}_{\xi}\|A\xi\|_2^2=\|A\|_F^2,
\]
which motivates
\begin{equation}
\mathcal{L}_{\mathrm{JCP}}
=
\mathbb{E}_{x,\xi}
\bigl\|
J_g(f_W(x))J_f(x)\xi-\xi
\bigr\|_2^2.
\label{eq:jcp}
\end{equation}
JCP is therefore not merely a representation regularizer: it is a probe-based objective that directly penalizes the local differential composition defect \(J_g(f_W(x))J_f(x)-I\), without explicitly forming the full Jacobian matrices.

When \(J_f(x)\) has full column rank, small composition defect implies that \(J_g(f_W(x))\) acts as an approximate left inverse of \(J_f(x)\). In the same full-column-rank least-squares setting, the Gauss--Newton direction is \(-J_f(x)^+r(x)\), where \(J_f(x)^+\) denotes the Moore--Penrose pseudoinverse \citep{penrose1955generalized}. Thus, when the learned reverse Jacobian approximates this local inverse action on the relevant residual directions, the induced D-IPG update becomes Gauss--Newton-like.

The same composition error defines a runtime diagnostic,
\begin{equation}
\mathrm{RJCP}(x)
:=
\mathbb{E}_{\xi}
\bigl\|
J_g(f_W(x))J_f(x)\xi-\xi
\bigr\|_2^2,
\label{eq:rjcp}
\end{equation}
which measures whether the learned local inverse geometry remains trustworthy along the optimization trajectory. JCP and RJCP are therefore two manifestations of the same geometric quantity: the former is minimized during training, while the latter monitors the achieved composition defect at inference time.

Finally, \eqref{eq:jcp} controls the average Jacobian-composition defect exactly in Frobenius norm. Writing
\[
A(x):=J_g(f_W(x))J_f(x)-I,
\]
Hutchinson's identity gives
\begin{equation}
\mathcal{L}_{\mathrm{JCP}}
=
\mathbb{E}_{x,\xi}\|A(x)\xi\|_2^2
=
\mathbb{E}_{x}\|A(x)\|_F^2.
\label{eq:jcp_frobenius}
\end{equation}
Since \(\|A(x)\|_2\le \|A(x)\|_F\), small JCP also upper-controls the average spectral composition defect appearing in the Gauss--Newton deviation bound of Section~\ref{sec:theory}.

\subsection{Practical realization}
\label{sec:practical_realization}

The framework is agnostic to the parameterization of \(f_W\) and \(g_V\). In low-dimensional settings, these maps may be represented by lightweight multilayer perceptrons; in spatial settings, they may be represented by shallow convolutional networks. What matters is not the architecture itself, but that \(f_W\) exposes the local forward geometry of the surrogate inverse problem and that \(g_V\) is trained to provide reverse differential action adapted to that geometry.

A useful training template is
\begin{equation}
\mathcal{L}
=
\lambda_{\mathrm{task}}\|f_W(x)-y\|_2^2
+
\lambda_{\mathrm{rec}}\|g_V(f_W(x))-x\|_2^2
+
\lambda_{\mathrm{cyc}}\|f_W(g_V(\widetilde y))-\widetilde y\|_2^2
+
\lambda_{\mathrm{JCP}}\,\mathcal{L}_{\mathrm{JCP}}
+
\mathcal{R}(W,V),
\label{eq:training_objective}
\end{equation}
where \((x,y)\) are paired latent-observation training samples, \(\widetilde y\) denotes measurement-space samples used for cycle consistency, and \(\mathcal{R}(W,V)\) collects optional stabilization terms. The exact coefficients, schedules, and active loss terms are reported in Appendix~\ref{app:experimental_setup} and Table~\ref{tab:hyperparameters}. Across benchmarks, the common design principle is fixed: task fit trains the forward surrogate, reconstruction and cycle-consistency terms anchor the reverse map when used, and JCP promotes the local differential consistency needed for inverse preconditioning.

All Jacobian-dependent quantities are implemented through automatic differentiation primitives rather than explicit Jacobian formation. JVPs compute \(J_f(x)\xi\) and \(J_g(f_W(x))(J_f(x)\xi)\), while reverse-mode differentiation backpropagates the resulting probe-based losses. This keeps JCP and RJCP substantially cheaper than explicit curvature-based linear solves and makes the method practical across both MLP- and CNN-based instantiations.

The relation to classical second-order optimization emerges from the induced first-order behavior of the D-IPG proposal. Expanding \(g_V\) around \(y_t=f_W(x_t)\) gives
\begin{equation}
g_V(y_t-\alpha_t r_t)
=
g_V(y_t)-\alpha_t J_g(y_t)r_t + O(\alpha_t^2\|r_t\|_2^2),
\label{eq:taylor}
\end{equation}
so that, under local consistency \(g_V(f_W(x_t))\approx x_t\),
\begin{equation}
x_{t+1}^{\mathrm{prop}}
\approx
x_t-\alpha_t J_g(f_W(x_t))r_t.
\label{eq:induced_step}
\end{equation}
Thus the implemented proposal induces an \(x\)-space preconditioned update whose preconditioning action is supplied by the learned reverse Jacobian. This first-order expansion is the starting point for the Gauss--Newton comparison in Section~\ref{sec:theory}: if \(J_g(f_W(x_t))\) approximates the local pseudoinverse action of \(J_f(x_t)\), then the D-IPG proposal approximates a damped Gauss--Newton step.

\section{Theory}
\label{sec:theory}

This section formalizes the local relation between D-IPG and Gauss--Newton. We show that the D-IPG proposal recovers a damped Gauss--Newton step to first order under local pseudoinverse consistency, and bound its deviation by the Jacobian-composition defect promoted by JCP. Proofs are provided in Appendix~\ref{app:proofs}.

\paragraph{First-order relation.}
\begin{theorem}[Local first-order equivalence to damped Gauss--Newton]
\label{thm:first_order_gn}
Let
\[
\Phi(x)=\frac{1}{2}\|f_W(x)-y^\star\|_2^2,
\qquad
y_t=f_W(x_t),
\qquad
r_t=y_t-y^\star,
\]
and consider the D-IPG proposal
\[
x_{t+1}^{\mathrm{prop}} = g_V(y_t-\alpha_t r_t).
\]
Assume that \(f_W\) is continuously differentiable near \(x_t\), that \(g_V\) is twice continuously differentiable near \(y_t\), and that \(J_f(x_t)\) has full column rank. Suppose further that
\[
g_V(f_W(x_t)) = x_t
\qquad\text{and}\qquad
J_g(f_W(x_t)) = J_f(x_t)^+ + E_t,
\]
where \(E_t\) is a perturbation matrix. Then
\begin{equation}
x_{t+1}^{\mathrm{prop}}
=
x_t-\alpha_t J_f(x_t)^+r_t-\alpha_t E_t r_t
+
O(\alpha_t^2\|r_t\|_2^2).
\label{eq:first_order_expansion}
\end{equation}
Consequently, the first-order discrepancy from the damped Gauss--Newton proposal satisfies
\[
\left\|
x_{t+1}^{\mathrm{prop}}
-
\left(x_t-\alpha_t J_f(x_t)^+r_t\right)
\right\|_2
\le
\alpha_t\|E_t r_t\|_2
+
O(\alpha_t^2\|r_t\|_2^2).
\]
Thus, when the learned inverse error \(E_t r_t\) is small on the residual direction, the D-IPG proposal is a first-order approximation to the damped Gauss--Newton step at \(x_t\).
\end{theorem}

Theorem~\ref{thm:first_order_gn} identifies the bridge between the learned solver and classical nonlinear least-squares. Although the proposal is written in measurement space, its first-order action is an \(x\)-space update preconditioned by the reverse Jacobian. When that Jacobian approximates the Moore--Penrose pseudoinverse of the forward Jacobian on the residual direction, D-IPG recovers the damped Gauss--Newton direction up to the explicit learned-inverse error and higher-order terms. The closest classical analogue is damped Gauss--Newton or LM at the level of direction geometry; the distinction is computational, since LM reconstructs and solves a damped linearized system at each iteration, whereas D-IPG amortizes the inverse action into \(g_V\).

\paragraph{Deviation bound.}
\begin{theorem}[Local deviation from damped Gauss--Newton]
\label{thm:deviation_gn}
Let
\[
J := J_f(x_t),
\qquad
G := J_g(f_W(x_t)),
\qquad
r_t := f_W(x_t)-y^\star,
\]
and assume that \(J\) has full column rank. Define the induced first-order D-IPG direction and the damped Gauss--Newton direction by
\[
\Delta x_{\mathrm{DIPG}} := -\alpha_t G r_t,
\qquad
\Delta x_{\mathrm{GN},\alpha} := -\alpha_t J^+ r_t.
\]
Assume further that the residual component under consideration lies in the range of \(J\), i.e.,
\[
r_t = Ju
\qquad
\text{for some } u\in\mathbb{R}^{d_{\mathrm{in}}}.
\]
Then
\begin{equation}
\|\Delta x_{\mathrm{DIPG}}-\Delta x_{\mathrm{GN},\alpha}\|_2
\le
\alpha_t
\frac{\|GJ-I\|_2}{\sigma_{\min}(J)}
\,
\|r_t\|_2,
\label{eq:main_deviation_bound}
\end{equation}
where \(\sigma_{\min}(J)\) is the smallest singular value of \(J\). Equivalently,
\begin{equation}
\|\Delta x_{\mathrm{DIPG}}-\Delta x_{\mathrm{GN},\alpha}\|_2
\le
\alpha_t
\frac{\|J_g(f_W(x_t))J_f(x_t)-I\|_2}{\sigma_{\min}(J_f(x_t))}
\,
\|r_t\|_2.
\label{eq:main_deviation_bound_expanded}
\end{equation}
\end{theorem}

Theorem~\ref{thm:deviation_gn} exposes the three controlling quantities: composition error, residual magnitude, and local conditioning. In the exact-consistency case \(J_g(f_W(x_t))J_f(x_t)=I\), the bound collapses to zero on the range of the local Jacobian.

\begin{corollary}[Exact local agreement]
\label{cor:exact}
Under the assumptions of Theorem~\ref{thm:deviation_gn}, if
\[
J_g(f_W(x_t))J_f(x_t)=I,
\]
then
\[
\Delta x_{\mathrm{DIPG}}=\Delta x_{\mathrm{GN},\alpha}
\]
for every residual \(r_t\in\mathrm{Range}(J_f(x_t))\).
\end{corollary}

\paragraph{Interpretation.}
In well-conditioned regimes, closeness to Gauss--Newton is the desired behavior: the local inverse problem is stable, the Gauss--Newton direction is informative, and small Jacobian-composition error implies that D-IPG inherits this trusted local geometry. The comparison becomes more delicate when the forward Jacobian is poorly conditioned, since the deviation bound scales like
\[
\frac{\|J_g(f_W(x_t))J_f(x_t)-I\|_2\,\|r_t\|_2}{\sigma_{\min}(J_f(x_t))}.
\]
Thus even modest composition error can induce larger directional discrepancy when \(\sigma_{\min}(J_f(x_t))\) is small. This need not be interpreted as a defect: in ill-conditioned regimes, controlled deviation from Gauss--Newton can act as an implicit regularization of unstable directions rather than a failure to approximate them.

The range-restricted form in Theorem~\ref{thm:deviation_gn} is natural for local nonlinear least-squares analysis, since Gauss--Newton acts on the component of the residual explained by the Jacobian linearization. A corresponding full-residual decomposition is given as Proposition~\ref{prop:full_residual} in Appendix~\ref{app:proofs}, separating \(r_t\) into a recoverable component \(r_{\parallel}\) and an orthogonal component \(r_{\perp}\). This distinction is also why the theory controls the latent-space composition \(GJ\), rather than requiring \(JG\approx I\) on all measurement-space residuals, since even exact Gauss--Newton only acts on the component of the residual explained by the local Jacobian.

Finally, JCP controls the same geometric quantity that appears in the Gauss--Newton deviation bound. Indeed, for
\[
A(x):=J_g(f_W(x))J_f(x)-I,
\]
the probe objective satisfies
\[
\mathcal{L}_{\mathrm{JCP}}
=
\mathbb{E}_{x,\xi}\|A(x)\xi\|_2^2
=
\mathbb{E}_{x}\|A(x)\|_F^2,
\]
and hence upper-controls \(\mathbb{E}_{x}\|A(x)\|_2^2\). Lemma~\ref{lem:jcp_composition} in Appendix~\ref{app:jcp_composition_lemma} gives the proof. In this sense, the training objective, the runtime diagnostic RJCP, and the Gauss--Newton deviation bound all track the same local inverse-consistency error.

Taken together, these results place D-IPG on the same local geometric footing as Gauss--Newton without requiring exact agreement with it. JCP promotes local inverse consistency during training, RJCP measures that consistency at runtime, and the deviation bounds clarify how the quality of the learned geometry interacts with conditioning and residual structure.

\section{Experiments}
\label{sec:exp}

We evaluate D-IPG on seven inverse-problem benchmarks spanning linear diffusion, nonlinear reaction--diffusion, elliptic inversion, advection--diffusion, and fluid dynamics. The main text focuses on three representative problems, Heat-3D, Advection--Diffusion-2D, and Allen--Cahn-2D, for which we show trajectory-level comparisons and pooled solver profiles. The complete seven-problem benchmark table, including the known Heat-1D failure mode that limits the full pooled asymptote, together with additional problem-specific summaries, is deferred to Appendix~\ref{app:extended_eval}. All methods are compared on matched instance sets under identical stopping budgets.

Figure~\ref{fig:performance} evaluates solver efficiency under the normalized residual tolerance
\[
r_t/r_0 \le 0.30,
\]
which is appropriate for time-to-tolerance and performance-profile comparisons. Figure~\ref{fig:mechanism} studies basin occupancy on Allen--Cahn-2D using the criterion
\[
\min_t \mathrm{RMSE}_t < 0.095,
\]
which isolates whether a trial reaches the low-error reconstruction basin. Figure~\ref{fig:reliability} reports task-level success rates under task-specific RMSE thresholds, and therefore measures reconstruction quality rather than optimization progress.

\subsection{Representative solver performance}

Figure~\ref{fig:performance} summarizes the main performance result on three representative PDE families. On Heat-3D, D-IPG and LM both descend rapidly within the first six iterations, whereas GN stalls at an elevated error plateau after an initially favorable descent. On Advection--Diffusion-2D, D-IPG reaches the residual tolerance in a median of 7 iterations with mean wall-clock time \(0.101\)s, while GN requires 33 median iterations and \(18.5\)s. On Allen--Cahn-2D, D-IPG again reaches tolerance in 6 median iterations and \(0.150\)s; GN attains the same iteration count but requires \(2.77\)s, while L-BFGS reaches low error only after substantially more iterations.

The bottom row of Figure~\ref{fig:performance} pools all 240 instances from these three representative problems into Dolan--Mor\'e profiles and a data profile \citep{dolan2002benchmarking}. The wall-clock profile shows that D-IPG is the fastest solver on nearly half of all pooled instances at performance ratio \(\tau=1\), and is the only method whose curve reaches \(1.0\) within the plotted range. The iteration profile shows a corresponding advantage in iteration efficiency, while the data profile shows that D-IPG reaches \(50\%\) solved at roughly \(10^{-1}\)s, whereas GN does not reach the same point until budgets exceeding \(1\)s. Thus, on the pooled representative benchmark, D-IPG achieves the best overall profile across wall-clock time, iteration count, and budgeted solve rate.

\begin{figure}[t]
    \centering
    \includegraphics[width=0.86\columnwidth]{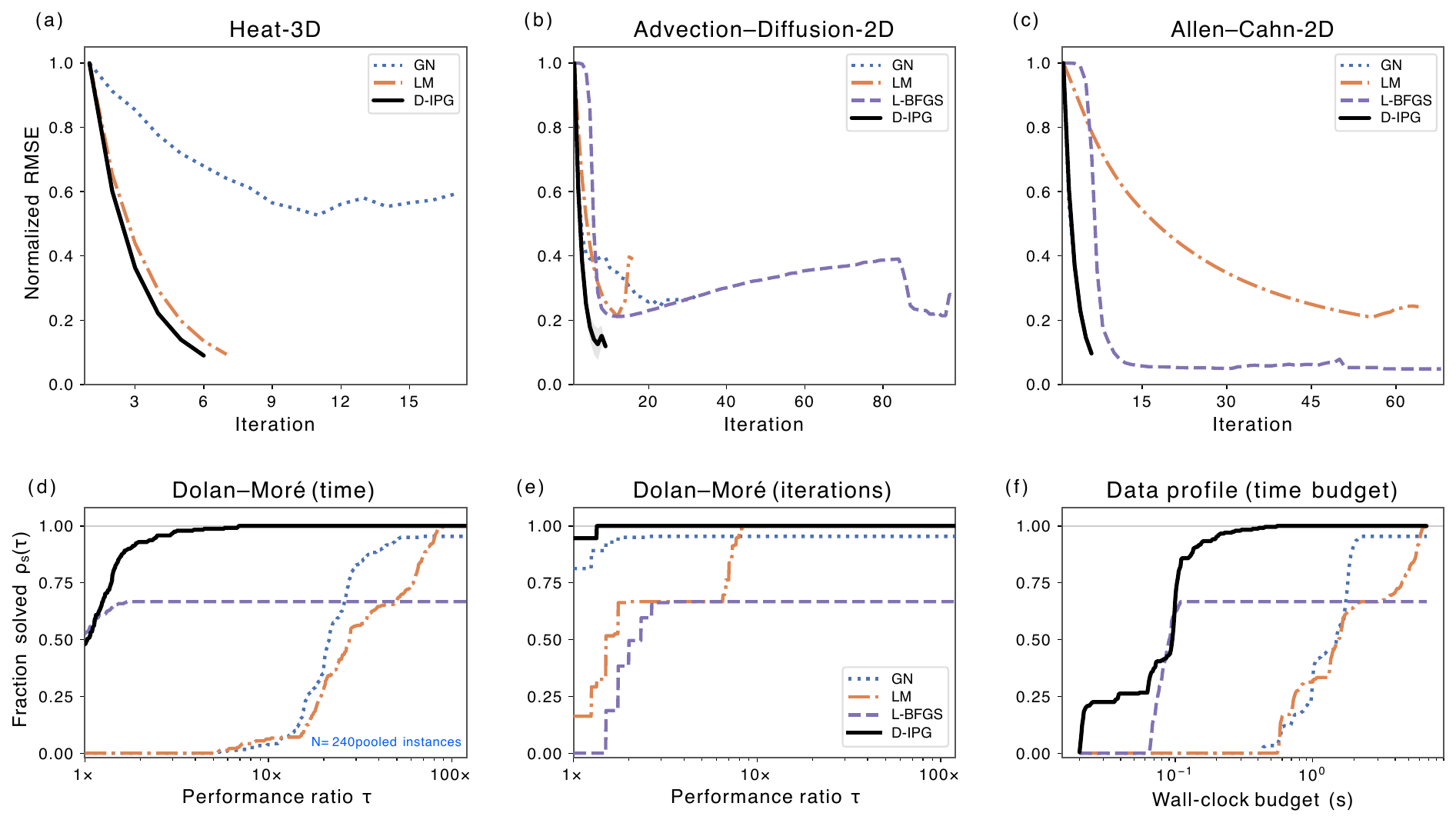}
    \caption{Representative performance on three PDE inverse problems. Top row: normalized RMSE trajectories, computed per trial as \(\mathrm{RMSE}_t/\mathrm{RMSE}_1\) and then averaged. Bottom row: pooled Dolan--Mor\'e performance profiles on wall-clock time and iteration count, together with a data profile over absolute wall-clock budgets. A trial is declared solved when \(r_t/r_0 \le 0.30\), so the pooled profiles measure time-to-tolerance rather than reconstruction quality.}
    \label{fig:performance}
\end{figure}

\subsection{Mechanism of JCP}

Figure~\ref{fig:mechanism} isolates the role of JCP on Allen--Cahn-2D. Panel~(a) plots the separation between converging and non-converging D-IPG\((-\mathrm{JCP})\) trials using Cohen's \(d\) \citep{cohen1988statistical}, with convergence defined by \(\min_t \mathrm{RMSE}_t < 0.095\). The key observation is that the \(-\mathrm{JCP}\) trials remain only weakly separated through the first five iterations, but split sharply at the final step: a small subset enters the low-error basin and continues contracting, while most trials settle into a higher-error basin. Panel~(b) shows that JCP largely suppresses this failure mode: the good basin is occupied by \(82.5\%\) of D-IPG\((+\mathrm{JCP})\) trials but only \(18.75\%\) of D-IPG\((-\mathrm{JCP})\) trials. Panel~(c) conditions on the successful D-IPG\((-\mathrm{JCP})\) subset and shows that, once \(-\mathrm{JCP}\) reaches the correct basin, its local contraction behavior is already close to that of D-IPG\((+\mathrm{JCP})\). Thus JCP improves reliability primarily by increasing access to the good basin, rather than by changing the local contraction dynamics after that basin is reached.

\begin{figure}[t]
    \centering
    \includegraphics[width=0.90\columnwidth]{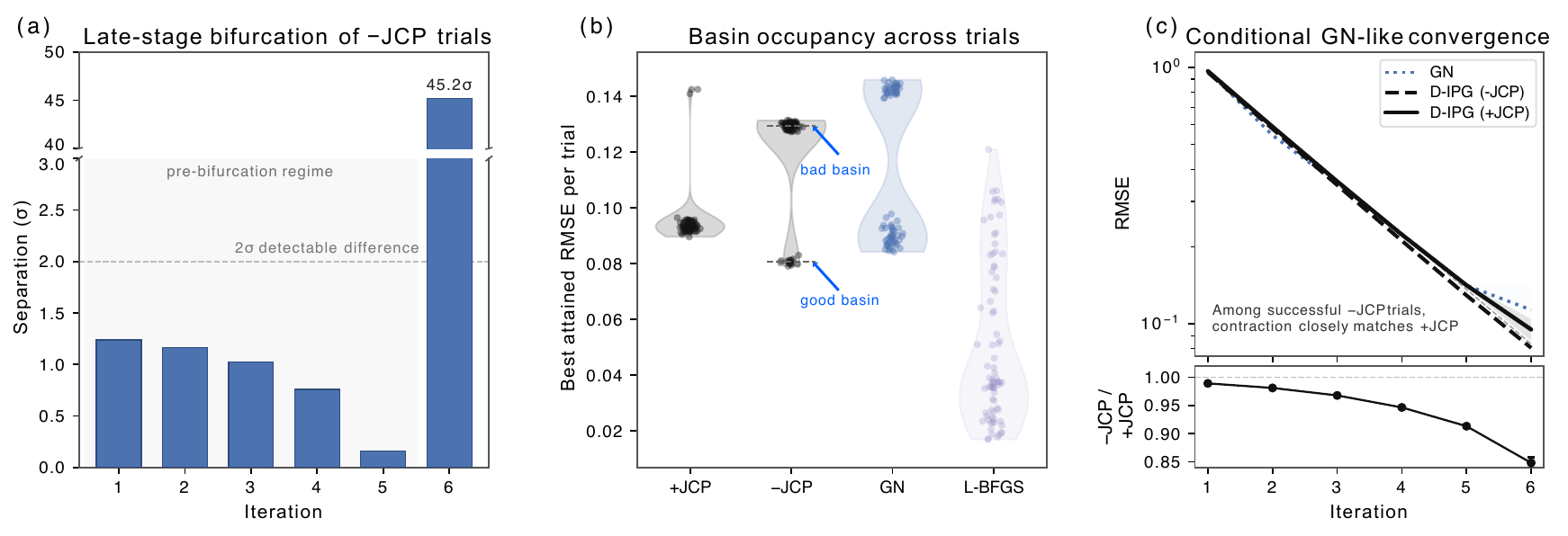}
    \caption{Mechanism study on Allen--Cahn-2D. Panel~(a) shows the iteration-wise separation between converging and non-converging D-IPG\((-\mathrm{JCP})\) trials, measured by Cohen's \(d\) under the basin criterion \(\min_t \mathrm{RMSE}_t < 0.095\). Panel~(b) shows the resulting basin occupancy across trials, with the same threshold defining the low-error basin. Panel~(c) compares D-IPG\((+\mathrm{JCP})\), GN, and the subset of D-IPG\((-\mathrm{JCP})\) trials that reach the good basin; the ratio strip reports the mean RMSE ratio of successful D-IPG\((-\mathrm{JCP})\) trials relative to D-IPG\((+\mathrm{JCP})\).}
    \label{fig:mechanism}
\end{figure}

\subsection{Reliability across the broader problem spectrum}

Figure~\ref{fig:reliability} studies reliability over six PDE inverse problems, ordered by the D-IPG\((+\mathrm{JCP})\) RJCP diagnostic as a descriptive proxy for learned local inverse quality. Each cell reports the task-specific RMSE-threshold success rate together with mean wall-clock time.

D-IPG attains \(100\%\) success on five of the six tested problems and \(69\%\) on Darcy-2D, producing an arithmetic mean success rate of \(94.8\%\) across the full benchmark. No baseline exhibits comparable uniformity: GN achieves only \(17.3\%\) average success, LM \(65.5\%\), and L-BFGS remains unreliable on several harder problems. The main exception, Darcy-2D, is consistent with Theorem~\ref{thm:deviation_gn}: Darcy inversion is an ill-conditioned elliptic problem, and its achieved RJCP is among the largest in the benchmark, so composition error and conditioning can jointly amplify deviation from the Gauss--Newton direction. Thus Darcy-2D is not a contradiction of the mechanism, but the regime where the theory predicts reduced reliability.

Crucially, on Navier--Stokes-2D, D-IPG maintains a \(100\%\) success rate despite an update direction nearly orthogonal to the negative gradient (cosine \(\approx 0.085\); Appendix Figure~\ref{fig:appendix_a2}), directly ruling out the interpretation that the method merely recovers first-order descent. 

Taken together, the experiments show that D-IPG is the strongest solver on the representative benchmark under wall-clock, iteration, and budgeted profiles; that JCP improves reliability primarily through basin access and stabilization of learned inverse geometry; and that this advantage extends to a broader problem spectrum where D-IPG is more reliable than the classical baselines considered here.

\begin{figure}[t]
    \centering
    \includegraphics[width=0.80\columnwidth]{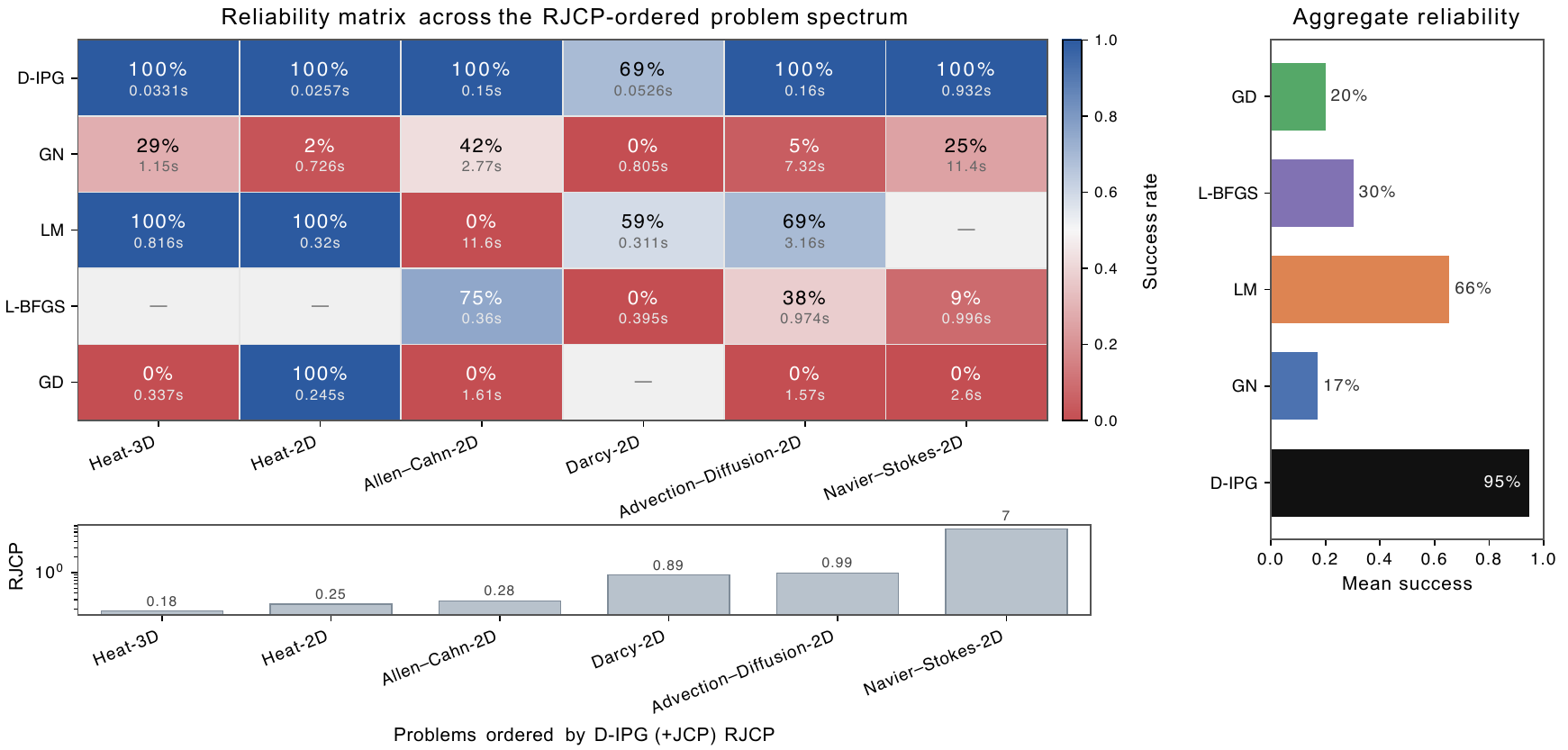}
    \caption{Reliability matrix across six PDE inverse problems. Columns are ordered by the D-IPG\((+\mathrm{JCP})\) RJCP diagnostic; the bottom strip reports the corresponding achieved RJCP values. Each heatmap cell shows the success rate under the task-specific RMSE quality threshold together with mean wall-clock time, while the right marginal reports the arithmetic mean success rate across problems.}
    \label{fig:reliability}
\end{figure}

\section{Conclusion}
\label{sec:conclusion}

This paper suggests that the expensive part of nonlinear inverse optimization is not only evaluating the forward model, but repeatedly rebuilding local inverse geometry. D-IPG shows that part of this geometry can be amortized into a learned reverse map and reused inside a safeguarded solver. In this view, JCP is the key structural ingredient: it trains the reverse Jacobian to carry local inverse action, allowing the resulting update to behave like an efficient Gauss--Newton-like proposal without solving a new Jacobian-based system at every iteration. The empirical results support this perspective across PDE inverse problems, while also showing that the advantage is strongest in amortized regimes where many instances share the same forward family.

\paragraph{Limitations.}
The method depends on the learned forward surrogate and reverse geometry: if \(f_W\) fails to expose reliable local inverse structure, as observed in the Heat-1D failure mode, neither JCP nor D-IPG can recover consistently reliable reverse updates. The Gauss--Newton interpretation is local and most meaningful when the surrogate linearization remains informative; global convergence would require additional descent-compatibility assumptions. JCP improves reliability mainly through local inverse consistency and basin access, but is not universally monotone across metrics or benchmarks. Finally, D-IPG is best suited to amortized settings with many inverse instances from the same forward family; for one-off problems, Jacobian-probe training may outweigh deployed solve-time gains.

\section*{Acknowledgments}
This work was conducted independently. Experiments were performed on Google Colab GPU instances. The author thanks the broader numerical-optimization and machine-learning research communities for the open literature that this work builds on.

\bibliographystyle{unsrtnat}
\bibliography{references}

\clearpage
\appendix

\section*{Code availability}
A code repository is available at
\url{https://github.com/AadityaKachhadiya/deceptron}.
The repository contains an installable PyTorch implementation of Deceptron/D-IPG, including the learned forward--reverse module, the D-IPG inference routine, and JCP/RJCP utilities. It also includes demonstration scripts, benchmark implementations for the full experimental suite, and figure-generation notebooks/scripts for the reported Dolan--Mor\'e profiles, basin analyses, reliability matrices, and appendix diagnostics. Setup instructions are provided, including editable installation via \texttt{pip install -e .}.

\section{Technical Proofs and Additional Theoretical Details}
\label{app:proofs}

This appendix provides the proofs of the main theoretical results in Section~\ref{sec:theory}. Throughout, vector norms are Euclidean and matrix norms are induced spectral norms unless stated otherwise.

\subsection{Proof of Theorem~\ref{thm:first_order_gn}}
\label{app:proof_first_order_gn}

\begin{proof}
Let \(y_t=f_W(x_t)\). By construction,
\[
x_{t+1}^{\mathrm{prop}} = g_V(y_t-\alpha_t r_t).
\]
Since \(g_V\) is \(C^2\) near \(y_t\), Taylor's theorem gives
\[
g_V(y_t-\alpha_t r_t)
=
g_V(y_t)-\alpha_t J_g(y_t)r_t+\varepsilon_t,
\]
where
\[
\|\varepsilon_t\|_2 \le C_t \alpha_t^2 \|r_t\|_2^2
\]
for some local constant \(C_t>0\). Since \(y_t=f_W(x_t)\), the local consistency assumption gives
\[
g_V(y_t)=g_V(f_W(x_t))=x_t.
\]
Using
\[
J_g(y_t)=J_g(f_W(x_t))=J_f(x_t)^+ + E_t,
\]
we obtain
\[
x_{t+1}^{\mathrm{prop}}
=
x_t-\alpha_t\bigl(J_f(x_t)^+ + E_t\bigr)r_t+\varepsilon_t.
\]
Therefore
\[
x_{t+1}^{\mathrm{prop}}
=
x_t-\alpha_t J_f(x_t)^+r_t-\alpha_t E_t r_t
+
O(\alpha_t^2\|r_t\|_2^2),
\]
which proves the main expansion.

Taking norms relative to the damped Gauss--Newton proposal gives
\[
\left\|
x_{t+1}^{\mathrm{prop}}
-
\left(x_t-\alpha_t J_f(x_t)^+r_t\right)
\right\|_2
\le
\alpha_t\|E_t r_t\|_2
+
\|\varepsilon_t\|_2.
\]
Using the Taylor remainder bound,
\[
\left\|
x_{t+1}^{\mathrm{prop}}
-
\left(x_t-\alpha_t J_f(x_t)^+r_t\right)
\right\|_2
\le
\alpha_t\|E_t r_t\|_2
+
O(\alpha_t^2\|r_t\|_2^2).
\]
Thus the D-IPG proposal agrees with the damped Gauss--Newton proposal up to the learned inverse error on the residual direction and higher-order Taylor terms.
\end{proof}

\subsection{Proof of Theorem~\ref{thm:deviation_gn}}
\label{app:proof_deviation_gn}

\begin{proof}
By definition,
\[
\Delta x_{\mathrm{DIPG}}-\Delta x_{\mathrm{GN},\alpha}
=
-\alpha_t G r_t + \alpha_t J^+ r_t
=
\alpha_t (J^+-G)r_t.
\]
Using \(r_t=Ju\), we obtain
\[
\Delta x_{\mathrm{DIPG}}-\Delta x_{\mathrm{GN},\alpha}
=
\alpha_t (J^+-G)Ju
=
\alpha_t (J^+J-GJ)u.
\]
Because \(J\) has full column rank,
\[
J^+J=I_{d_{\mathrm{in}}},
\]
so
\[
\Delta x_{\mathrm{DIPG}}-\Delta x_{\mathrm{GN},\alpha}
=
\alpha_t (I-GJ)u.
\]
Taking norms and using submultiplicativity gives
\begin{equation}
\|\Delta x_{\mathrm{DIPG}}-\Delta x_{\mathrm{GN},\alpha}\|_2
\le
\alpha_t \|I-GJ\|_2 \|u\|_2.
\label{eq:pre_sigma_bound_appendix}
\end{equation}
Since \(r_t=Ju\) and \(J\) has full column rank,
\[
\|Ju\|_2 \ge \sigma_{\min}(J)\|u\|_2,
\]
and therefore
\[
\|u\|_2 \le \frac{\|r_t\|_2}{\sigma_{\min}(J)}.
\]
Substituting this into \eqref{eq:pre_sigma_bound_appendix} gives
\[
\|\Delta x_{\mathrm{DIPG}}-\Delta x_{\mathrm{GN},\alpha}\|_2
\le
\alpha_t
\frac{\|I-GJ\|_2}{\sigma_{\min}(J)}
\,
\|r_t\|_2.
\]
Since \(\|I-GJ\|_2=\|GJ-I\|_2\), this proves
\[
\|\Delta x_{\mathrm{DIPG}}-\Delta x_{\mathrm{GN},\alpha}\|_2
\le
\alpha_t
\frac{\|GJ-I\|_2}{\sigma_{\min}(J)}
\,
\|r_t\|_2.
\]
Replacing \(J\) and \(G\) by \(J_f(x_t)\) and \(J_g(f_W(x_t))\) gives the expanded form.
\end{proof}

\subsection{Proof of Corollary~\ref{cor:exact}}
\label{app:proof_cor_exact}

\begin{proof}
Under the assumptions of Theorem~\ref{thm:deviation_gn}, the deviation bound gives
\[
\|\Delta x_{\mathrm{DIPG}}-\Delta x_{\mathrm{GN},\alpha}\|_2
\le
\alpha_t
\frac{\|J_g(f_W(x_t))J_f(x_t)-I\|_2}{\sigma_{\min}(J_f(x_t))}
\,
\|r_t\|_2.
\]
If \(J_g(f_W(x_t))J_f(x_t)=I\), then the numerator vanishes, hence
\[
\|\Delta x_{\mathrm{DIPG}}-\Delta x_{\mathrm{GN},\alpha}\|_2=0.
\]
Therefore
\[
\Delta x_{\mathrm{DIPG}}=\Delta x_{\mathrm{GN},\alpha}
\]
for every \(r_t\in \mathrm{Range}(J_f(x_t))\).
\end{proof}

\subsection{JCP as an estimator of composition defect}
\label{app:jcp_composition_lemma}

\begin{lemma}[JCP controls average composition defect]
\label{lem:jcp_composition}
Let
\[
A(x):=J_g(f_W(x))J_f(x)-I.
\]
Assume that the probe vector \(\xi\in\mathbb{R}^{d_{\mathrm{in}}}\) satisfies
\[
\mathbb{E}_{\xi}[\xi\xi^\top]=I.
\]
Then
\[
\mathbb{E}_{\xi}\|A(x)\xi\|_2^2=\|A(x)\|_F^2.
\]
Consequently,
\[
\mathcal{L}_{\mathrm{JCP}}
=
\mathbb{E}_{x}\|J_g(f_W(x))J_f(x)-I\|_F^2.
\]
In particular,
\[
\mathbb{E}_{x}\|J_g(f_W(x))J_f(x)-I\|_2^2
\le
\mathcal{L}_{\mathrm{JCP}}.
\]
\end{lemma}

\begin{proof}
For fixed \(x\), write \(A=A(x)\). Then
\[
\mathbb{E}_{\xi}\|A\xi\|_2^2
=
\mathbb{E}_{\xi}\left[\xi^\top A^\top A\xi\right].
\]
Using the trace identity for scalars,
\[
\xi^\top A^\top A\xi
=
\mathrm{tr}\!\left(\xi^\top A^\top A\xi\right)
=
\mathrm{tr}\!\left(A^\top A\xi\xi^\top\right).
\]
Taking expectation gives
\[
\mathbb{E}_{\xi}\|A\xi\|_2^2
=
\mathrm{tr}\!\left(A^\top A\,\mathbb{E}_{\xi}[\xi\xi^\top]\right)
=
\mathrm{tr}(A^\top A)
=
\|A\|_F^2.
\]
Averaging over \(x\) yields the expression for \(\mathcal{L}_{\mathrm{JCP}}\). Finally, since
\[
\|A(x)\|_2\le \|A(x)\|_F
\]
for every \(x\),
\[
\mathbb{E}_{x}\|A(x)\|_2^2
\le
\mathbb{E}_{x}\|A(x)\|_F^2
=
\mathcal{L}_{\mathrm{JCP}}.
\]
\end{proof}

\subsection{Full residual decomposition}
\label{app:full_residual_decomposition}

\begin{proposition}[Full residual decomposition]
\label{prop:full_residual}
Let
\[
J:=J_f(x_t),
\qquad
G:=J_g(f_W(x_t)),
\qquad
r_t:=f_W(x_t)-y^\star,
\]
and assume that \(J\) has full column rank. For an arbitrary residual \(r_t\), decompose
\[
r_t = r_{\parallel}+r_{\perp},
\qquad
r_{\parallel}\in \mathrm{Range}(J),
\qquad
r_{\perp}\perp \mathrm{Range}(J).
\]
Then
\begin{equation}
\|\Delta x_{\mathrm{DIPG}}-\Delta x_{\mathrm{GN},\alpha}\|_2
\le
\alpha_t
\left(
\frac{\|GJ-I\|_2}{\sigma_{\min}(J)}\|r_{\parallel}\|_2
+
\|G\|_2\,\|r_{\perp}\|_2
\right).
\label{eq:full_residual_bound}
\end{equation}
\end{proposition}

\begin{proof}
Decompose
\[
r_t=r_{\parallel}+r_{\perp}
\]
with \(r_{\parallel}\in\mathrm{Range}(J)\) and \(r_{\perp}\perp\mathrm{Range}(J)\). Then
\[
\Delta x_{\mathrm{DIPG}}-\Delta x_{\mathrm{GN},\alpha}
=
-\alpha_t G(r_{\parallel}+r_{\perp})+\alpha_t J^+(r_{\parallel}+r_{\perp}).
\]
Since \(r_{\perp}\perp \mathrm{Range}(J)\), one has \(J^+r_{\perp}=0\). Hence
\[
\Delta x_{\mathrm{DIPG}}-\Delta x_{\mathrm{GN},\alpha}
=
\alpha_t(J^+-G)r_{\parallel}-\alpha_t G r_{\perp}.
\]
Taking norms and applying the triangle inequality gives
\[
\|\Delta x_{\mathrm{DIPG}}-\Delta x_{\mathrm{GN},\alpha}\|_2
\le
\alpha_t\|(J^+-G)r_{\parallel}\|_2
+
\alpha_t\|G r_{\perp}\|_2.
\]
Because \(r_{\parallel}\in \mathrm{Range}(J)\), there exists \(u\) such that \(r_{\parallel}=Ju\). Repeating the argument from the proof of Theorem~\ref{thm:deviation_gn},
\[
\|(J^+-G)r_{\parallel}\|_2
=
\|(I-GJ)u\|_2
\le
\frac{\|GJ-I\|_2}{\sigma_{\min}(J)}\|r_{\parallel}\|_2.
\]
Also,
\[
\|G r_{\perp}\|_2 \le \|G\|_2\|r_{\perp}\|_2.
\]
Combining the two bounds yields \eqref{eq:full_residual_bound}.
\end{proof}

\subsection{Additional discussion on the range-restricted regime}
\label{app:range_restricted_discussion}

Given the local Jacobian \(J=J_f(x_t)\), any residual \(r_t\in\mathbb{R}^{d_{\mathrm{out}}}\) admits the orthogonal decomposition
\[
r_t = r_{\parallel}+r_{\perp},
\qquad
r_{\parallel}\in\mathrm{Range}(J),
\qquad
r_{\perp}\in\mathrm{Range}(J)^\perp.
\]
Theorem~\ref{thm:deviation_gn} applies to the component \(r_{\parallel}\), which is precisely the part explained by the local linearization of the forward model. Writing \(r_{\parallel}=Ju\), the bound quantifies how accurately D-IPG reproduces the damped Gauss--Newton direction on the locally recoverable component of the residual. The orthogonal component \(r_{\perp}\) is invisible to the pseudoinverse of \(J\) and therefore lies outside the scope of the local Gauss--Newton comparison. This is why the theorem is naturally local and range-restricted, rather than a global statement about arbitrary residuals.

\section{Extended Empirical Evaluation and Diagnostic Analyses}
\label{app:extended_eval}

This appendix provides the empirical record supporting the main-text claims. Its role is fourfold. First, it establishes that the representative problems shown in the main paper are not cherry-picked by extending the analysis to the full benchmark suite. Second, it sharpens the geometric interpretation of the learned pullback by isolating diagnostics that do not collapse to mean performance alone. Third, it documents the scope of the Jacobian Composition Penalty (JCP), including regimes where its effect is large, negligible, or mildly restrictive. Fourth, it reports hyperparameter sensitivity, qualitative reconstructions, robustness summaries, and amortization measurements in an auditable format.

Throughout the appendix, we preserve the same distinction between evaluation criteria used in the main text. Performance profiles and ECDF-style solver comparisons use the normalized residual tolerance
\[
r_t/r_0 \le 0.30.
\]
Mechanistic basin analyses on Allen--Cahn-2D use the good-basin criterion
\[
\min_t \mathrm{RMSE}_t < 0.095.
\]
Benchmark success rates reported in the summary tables use the task-specific RMSE quality threshold recorded in the corresponding summary CSV for each problem. This separation is essential because a method may reduce the residual sufficiently quickly while still failing to achieve satisfactory reconstruction quality.

\paragraph{Pooled benchmark view across the full instance set.}
Figure~\ref{fig:appendix_a1} extends the performance-profile analysis of the main paper to the full pooled benchmark. The three panels use the same visual vocabulary as Figure~1 in the main text: a Dolan--Mor\'e profile on wall-clock time, a Dolan--Mor\'e profile on iterations, and a data profile under an absolute wall-clock budget. The point of this figure is not to replace the per-problem analysis, but to show how relative solver rankings behave once all available instances are aggregated.

A crucial point is that the pooled asymptote of D-IPG is limited by Heat-1D. Heat-1D contributes \(300\) instances and D-IPG succeeds on only \(24.3\%\) of them under the task-specific RMSE success criterion. Heat-1D should therefore be interpreted as a specific failure mode rather than as evidence against the broader mechanism. In this benchmark, the learned reverse operator does not achieve sufficiently reliable local inverse structure, and the resulting failures dominate the pooled asymptote because the problem contributes many instances. This is why we isolate Heat-1D explicitly rather than allowing it to be hidden inside the pooled benchmark. Excluding Heat-1D, D-IPG succeeds on all remaining pooled instances in the current benchmark.

\begin{figure*}[t]
    \centering
    \includegraphics[width=\textwidth]{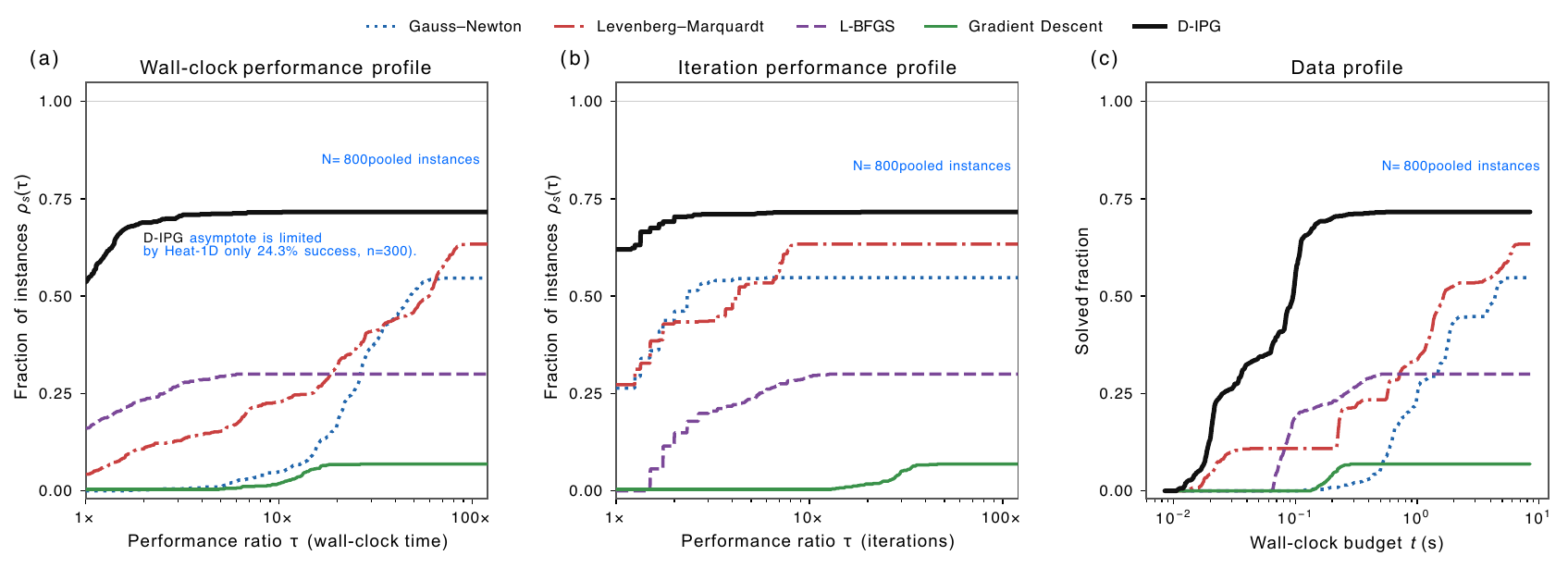}
    \caption{Pooled solver profiles across the full benchmark. (a) wall-clock Dolan--Mor\'e profile. (b) iteration Dolan--Mor\'e profile. (c) data profile under an absolute wall-clock budget. All three panels use the normalized residual tolerance criterion \(r_t/r_0 \le 0.30\). The D-IPG asymptote is limited primarily by Heat-1D (\(24.3\%\) success over \(n=300\) instances), which is a known failure mode rather than a contradiction of the main-text results. Excluding Heat-1D, D-IPG attains perfect success on the remaining pooled benchmark instances.}
    \label{fig:appendix_a1}
\end{figure*}

\paragraph{Update geometry across problems.}
Figure~\ref{fig:appendix_a2} reports the cosine similarity between the D-IPG update direction and the negative gradient across problems where this diagnostic is available. The point of this figure is not to claim full benchmark completeness, but to show that the learned pullback does not collapse to a single geometric regime. Some successful problems are nearly gradient-aligned, while others are strongly corrective.

The most informative case is Navier--Stokes-2D, where D-IPG attains \(100\%\) success despite a cosine similarity close to zero. This shows that the method is not merely recovering a rescaled negative-gradient direction. Instead, the learned pullback can generate corrective updates that are geometrically distinct from first-order descent while remaining effective. At the other extreme, Heat-3D is strongly gradient-aligned and also succeeds, so the relevant conclusion is not that one geometric regime is universally better, but that D-IPG remains effective across qualitatively different update geometries. Heat-1D provides the counterexample: its intermediate cosine and low success rate show that geometric deviation alone does not determine reliability.

\begin{figure}[t]
    \centering
    \includegraphics[width=0.55\columnwidth]{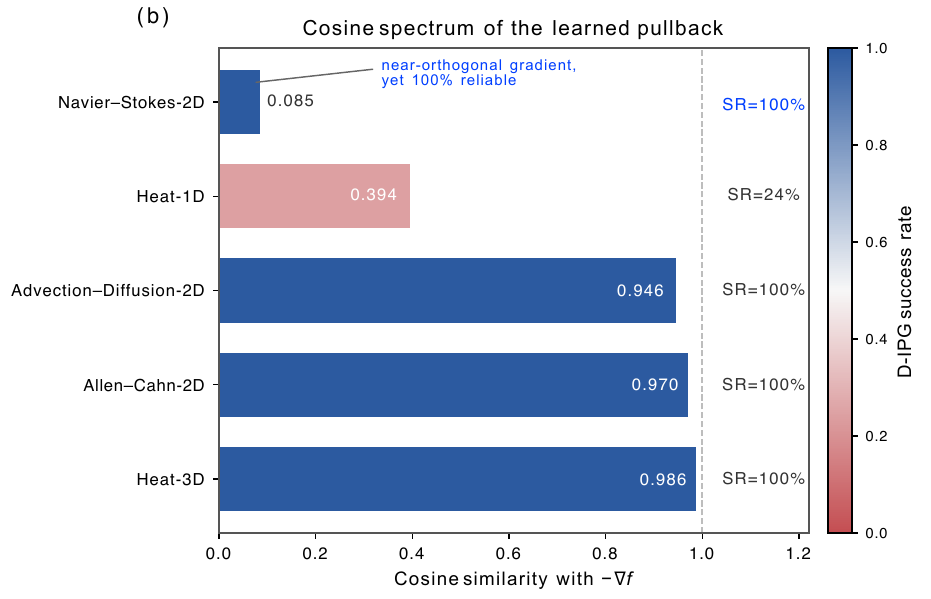}
    \caption{Cosine spectrum of the learned pullback across problems with available cosine diagnostics. Bar length gives the cosine similarity between the D-IPG update and \(-\nabla f\); bar color encodes the D-IPG success rate under the task-specific RMSE criterion. The figure should be read as a geometric spectrum rather than a complete benchmark summary. In particular, Navier--Stokes-2D is nearly orthogonal to the raw gradient direction yet remains fully reliable, showing that D-IPG is not simply learning gradient descent.}
    \label{fig:appendix_a2}
\end{figure}

\paragraph{Reproducibility geometry under JCP.}
Figure~\ref{fig:appendix_a3} isolates a geometric effect of JCP that is distinct from mean speed or mean RMSE. Panel~(a) shows that the coefficient of variation of final RMSE across trials is monotonically associated with final RJCP across all seven problems and both JCP conditions. The appropriate statistical claim is therefore monotone association rather than a precisely estimated functional law. We report the nonparametric rank correlation
\[
\rho_{\mathrm{Spearman}} = 0.807,\qquad p < 0.001,\qquad N = 14,
\]
which is robust to the limited sample size and does not require distributional assumptions. The safe conclusion is that larger composition error is associated with less reproducible final outcomes.

Panel~(b) shows that this effect remains visible even on a problem where both conditions achieve \(100\%\) success. On Navier--Stokes-2D, the standard deviation across trials of \(\log(r_t/r_1)\) is systematically smaller under \(+\)JCP than under \(-\)JCP. At iteration \(6\), the \(-\)JCP bundle is approximately \(36\%\) wider. Thus JCP is not merely rescuing failed runs on a hard basin-structured problem; it also tightens the convergence geometry on problems where success rates are already saturated. This supports the interpretation that JCP improves the reproducibility of the learned pullback. Its role is not to enforce monotone improvement in every scalar diagnostic, but to improve the reliability of the induced inverse geometry.

\begin{figure*}[t]
    \centering
    \includegraphics[width=0.80\textwidth]{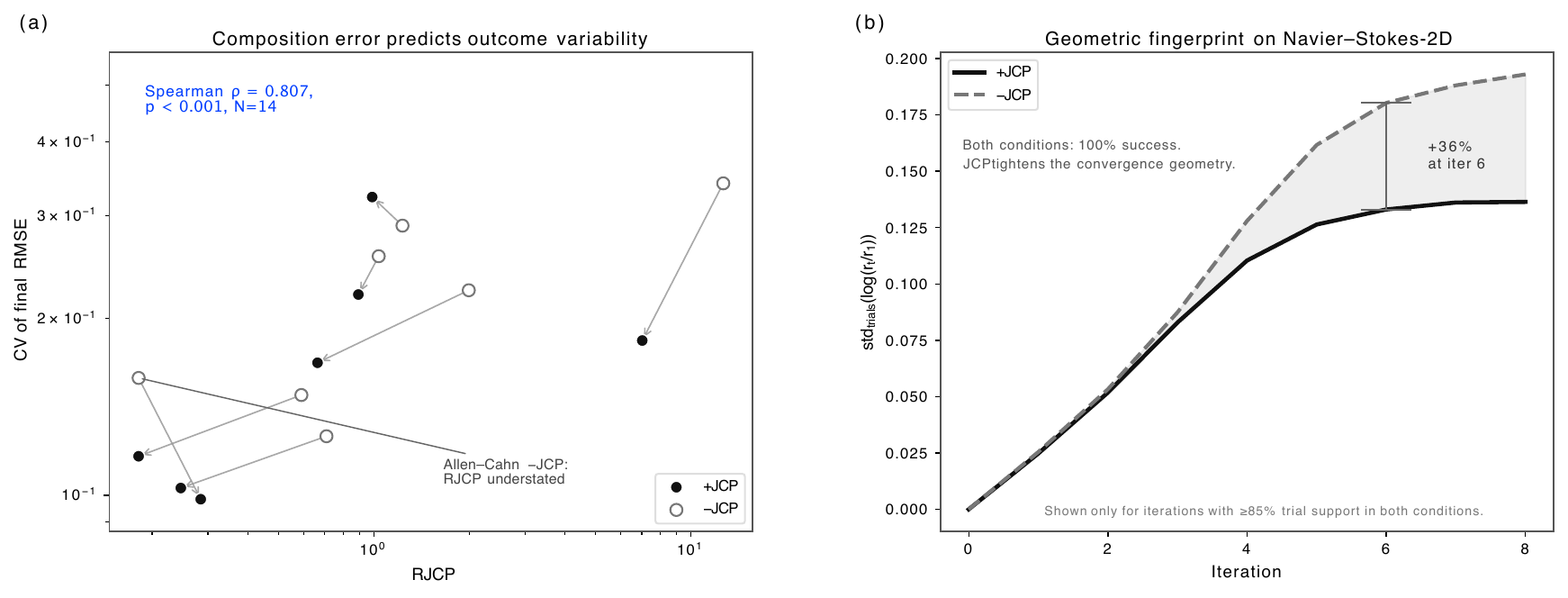}
    \caption{Reproducibility effects of the Jacobian Composition Penalty. Left: cross-problem association between final RJCP and the coefficient of variation of final RMSE, using both JCP conditions across seven PDE families. Filled markers denote \(+\)JCP and open markers denote \(-\)JCP; arrows connect paired conditions for the same problem. The appropriate statistical claim is monotone association, quantified by Spearman \(\rho=0.807\), \(p<0.001\), \(N=14\). Right: trajectory-bundle width on Navier--Stokes-2D, measured as \(\mathrm{std}_{\mathrm{trials}}(\log(r_t/r_1))\). Both conditions achieve \(100\%\) success, yet \(+\)JCP yields a visibly tighter bundle; at iteration \(6\), the \(-\)JCP bundle is approximately \(36\%\) wider. Shown iterations are restricted to those with at least \(85\%\) trial support in both conditions to avoid unstable tail effects.}
    \label{fig:appendix_a3}
\end{figure*}

\paragraph{Hyperparameter analysis on Allen--Cahn-2D.}
Figure~\ref{fig:appendix_a4} reports two controlled ablations on Allen--Cahn-2D. The left column varies the number of Hutchinson probes used in the JCP estimator. The success rate is non-monotone: performance improves from \(k=1\) to \(k=4\), peaks near \(96.25\%\), and then degrades at \(k=8\). The \(-\)JCP curve remains nearly flat, showing that probe count affects estimator quality only when the JCP term is present. The validation RJCP curve exhibits a low-RJCP regime around \(k=2\) to \(4\), consistent with a bias--variance tradeoff rather than a monotone benefit from more probes.

The right column varies the D-IPG step size. The trajectory panel shows qualitatively distinct failure modes: small step sizes stagnate after long runs, while excessively large step sizes descend quickly but settle into the wrong basin. Among the tested values, only \(\alpha=1.0\) attains high success (\(96.25\%\)). This is an honest limitation, and it is structurally interpretable: since \(y^{\mathrm{prop}}=y_t-\alpha r_t\), the value \(\alpha=1\) sends the proposal to \(y^\star\). Smaller values under-correct in measurement space, while larger values extrapolate past the target; on basin-structured problems such as Allen--Cahn-2D, this can determine whether the reverse map lands in the correct basin.

\begin{figure*}[t]
    \centering
    \includegraphics[width=0.76\textwidth]{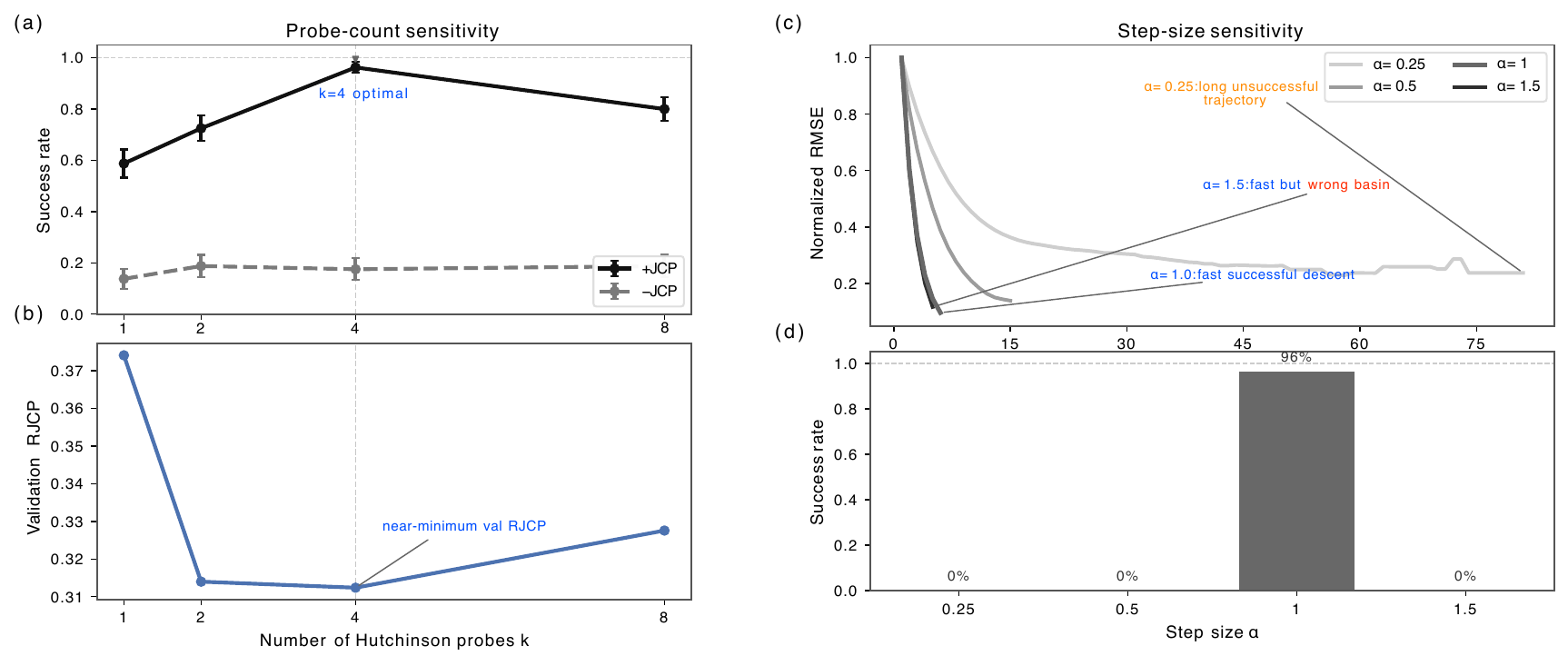}
    \caption{Hyperparameter analysis on Allen--Cahn-2D. Left: probe-count sensitivity. Top: success rate versus number of Hutchinson probes \(k\) for \(+\)JCP and \(-\)JCP. Bottom: validation RJCP versus \(k\), showing a low-RJCP regime around \(k=2\) to \(4\). Right: step-size sensitivity. Top: normalized RMSE trajectories for different step sizes \(\alpha\), revealing distinct failure modes (slow stagnation for small \(\alpha\), fast but incorrect descent for large \(\alpha\)). Bottom: success rates under the task-specific RMSE criterion. Only \(\alpha=1.0\) attains high success, indicating that Armijo backtracking helps but does not remove step-size sensitivity.}
    \label{fig:appendix_a4}
\end{figure*}

\paragraph{Qualitative reconstruction examples.}
Figure~\ref{fig:appendix_reconstruction} provides representative reconstructions for the three main-text problems. These plots are not intended as a substitute for the quantitative benchmark tables; rather, they show the recovered fields directly and help relate the optimization diagnostics to the visual quality of the reconstructed states. This figure is particularly useful for inverse-problem audiences, who often want to confirm that numerical improvements translate into visibly improved reconstructions.

The examples are consistent with the quantitative results. D-IPG(+JCP) recovers the underlying structures cleanly across all three showcased problems, while the strongest baselines exhibit method-specific degradation modes, including oversmoothing, noisy artifacts, or incomplete recovery of the target geometry. We therefore treat this figure as qualitative corroboration, not as independent evidence of superiority.

\begin{figure*}[t]
    \centering
    \includegraphics[width=0.76\textwidth]{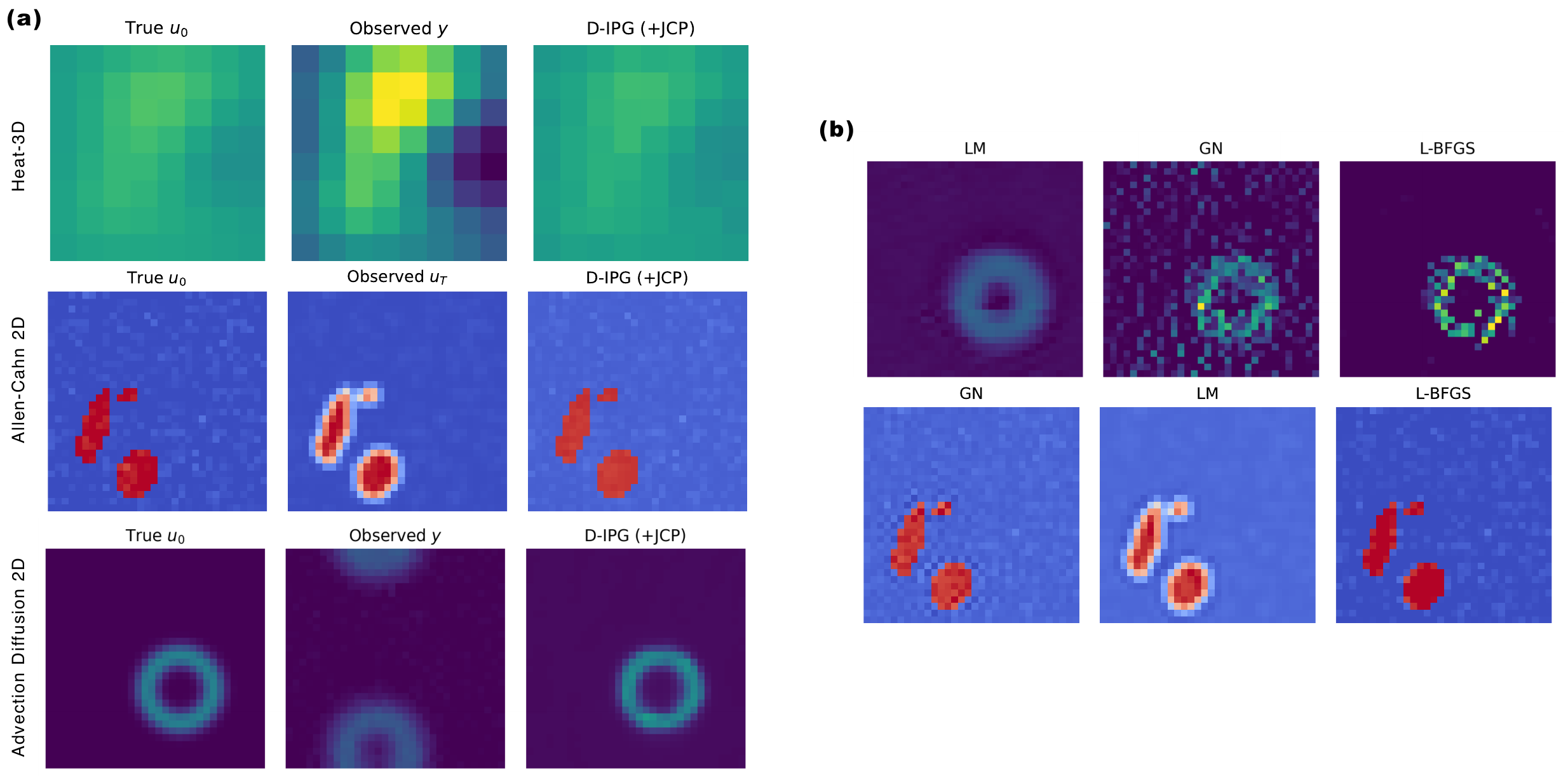}
    \caption{Representative qualitative reconstructions on three benchmark problems. Left block: ground truth, observed measurement, and D-IPG(+JCP) reconstruction. Right block: corresponding baseline reconstructions for the same instances. The figure is intended as qualitative corroboration of the quantitative benchmark and mechanism analyses, showing that the improvements reported in the main text translate into visibly cleaner recovered fields.}
    \label{fig:appendix_reconstruction}
\end{figure*}

\paragraph{Appendix tables.}
To make the appendix auditable, we include compact tables exposing the exact metrics underlying the main-text and appendix figures. Table~\ref{tab:metric_definitions} summarizes evaluation criteria. Table~\ref{tab:main_text_problem_full_results} gives complete method-level results for the three representative main-text problems. Table~\ref{tab:full_benchmark_summary} summarizes the full seven-problem benchmark. Table~\ref{tab:jcp_head_to_head} isolates the effect of JCP at the per-problem level. Table~\ref{tab:per_iteration_operations} summarizes inference-time per-iteration structure. Table~\ref{tab:seed_robustness_compact} reports seed robustness, and Table~\ref{tab:amortization_break_even} reports amortized training-cost break-even points.

\begin{table}[t]
\centering
\caption{Evaluation criteria used across the main paper and appendix.}
\label{tab:metric_definitions}
\small
\begin{tabular}{ll}
\toprule
Location & Criterion \\
\midrule
Figure~\ref{fig:performance} and Figure~\ref{fig:appendix_a1} 
& solved if $r/r_0 \le 0.30$ \\
Figure~\ref{fig:mechanism} 
& good basin if $\min_t \mathrm{RMSE}_t < 0.095$ on Allen--Cahn-2D \\
Figure~\ref{fig:reliability} and benchmark success tables 
& task-specific RMSE threshold from summary CSVs \\
Figure~\ref{fig:appendix_a3}(a) 
& coefficient of variation of final RMSE across trials \\
Figure~\ref{fig:appendix_a3}(b) 
& $\mathrm{std}_{\mathrm{trials}}(\log(r_t/r_1))$ \\
\bottomrule
\end{tabular}
\end{table}

\begin{table*}[t]
\centering
\caption{Complete method-level results for the three representative main-text problems. Success rate uses the task-specific RMSE quality threshold from the summary CSVs. RJCP is reported only for D-IPG variants. Classical baselines use the forward surrogate only and are therefore reported once per problem.}
\label{tab:main_text_problem_full_results}
\resizebox{\textwidth}{!}{
\begin{tabular}{llccccc}
\toprule
Problem & Method / Ablation & Success Rate & Median Iters & Mean RMSE & Mean Time (s) & Mean Final RJCP \\
\midrule
\multirow{5}{*}{Heat-3D}
& D-IPG (+JCP)     & 1.0000 & 5  & 0.063546 & 0.033150 & 0.181220 \\
& D-IPG (-JCP)     & 1.0000 & 5  & 0.057742 & 0.030329 & 0.590936 \\
& Gradient Descent & 0.0000 & 80 & 0.124150 & 0.337239 & ---      \\
& GN               & 0.2875 & 7  & 0.333657 & 1.149244 & ---      \\
& LM               & 1.0000 & 6  & 0.064407 & 0.816331 & ---      \\
\midrule
\multirow{6}{*}{Advection--Diffusion-2D}
& D-IPG (+JCP)     & 1.0000 & 7  & 0.054671 & 0.159624 & 0.988080 \\
& D-IPG (-JCP)     & 1.0000 & 6  & 0.050482 & 0.153320 & 1.230998 \\
& Gradient Descent & 0.0000 & 80 & 0.292912 & 1.573286 & ---      \\
& GN               & 0.0500 & 14 & 0.160834 & 7.319908 & ---      \\
& LM               & 0.6875 & 12 & 0.087936 & 3.161762 & ---      \\
& L-BFGS           & 0.3750 & 84 & 0.146741 & 0.974034 & ---      \\
\midrule
\multirow{6}{*}{Allen--Cahn-2D}
& D-IPG (+JCP)     & 1.0000 & 6  & 0.093849 & 0.149842 & 0.284511 \\
& D-IPG (-JCP)     & 0.1625 & 5  & 0.121996 & 0.122813 & 0.181261 \\
& Gradient Descent & 0.0000 & 80 & 0.930444 & 1.614167 & ---      \\
& GN               & 0.4250 & 6  & 0.100346 & 2.765452 & ---      \\
& LM               & 0.0000 & 55 & 0.197546 & 11.622093 & ---      \\
& L-BFGS           & 0.7500 & 31 & 0.076070 & 0.359573 & ---      \\
\bottomrule
\end{tabular}
}
\end{table*}

\begin{sidewaystable}[p]
\centering
\caption{Full benchmark summary across all seven PDE inverse problems. D-IPG(+JCP) is the proposed method. Best-baseline columns report the strongest non-D-IPG method under the task-specific RMSE success criterion. Blank RMSE entries indicate that the corresponding aggregate RMSE was not available in the same summary table and is therefore left unspecified rather than imputed.}
\label{tab:full_benchmark_summary}
\small
\setlength{\tabcolsep}{6pt}
\begin{tabular}{lcccccccc}
\toprule
Problem & D-IPG(+JCP) SR & D-IPG(-JCP) SR & Best baseline & Best-baseline SR & D-IPG RMSE & Best-baseline RMSE & D-IPG median iters & D-IPG time (s) \\
\midrule
Heat-1D & 0.243 & 0.200 & Gradient Descent & 0.133 & 0.2704 & 0.3219 & 35.0 & 0.068 \\
Heat-2D & 1.000 & 1.000 & Gradient Descent & 1.000 & 0.0810 & 0.1204 & 5.0 & 0.026 \\
Heat-3D & 1.000 & 1.000 & LM & 1.000 & 0.0635 & 0.0644 & 5.0 & 0.033 \\
Darcy-2D & 0.688 & 0.662 & LM & 0.590 & 0.1659 & --- & 2.0 & 0.053 \\
Advection--Diffusion-2D & 1.000 & 1.000 & LM & 0.6875 & 0.0547 & 0.0879 & 7.0 & 0.160 \\
Allen--Cahn-2D & 1.000 & 0.163 & L-BFGS & 0.750 & 0.0938 & 0.0761 & 6.0 & 0.150 \\
Navier--Stokes-2D & 1.000 & 1.000 & GN & 0.250 & 0.0532 & --- & 30.0 & 0.932 \\
\bottomrule
\end{tabular}
\end{sidewaystable}

\begin{table*}[t]
\centering
\caption{Per-problem effect of JCP on D-IPG. Positive $\Delta$SR indicates a reliability gain from the Jacobian Composition Penalty. Blank entries indicate diagnostics not recorded in the available summary files and are intentionally left unspecified rather than imputed.}
\label{tab:jcp_head_to_head}
\resizebox{\textwidth}{!}{
\begin{tabular}{lccccccccc}
\toprule
Problem & SR(+JCP) & SR(-JCP) & $\Delta$SR & RMSE(+JCP) & RMSE(-JCP) & RJCP(+JCP) & RJCP(-JCP) & Cosine(+JCP) & Cosine(-JCP) \\
\midrule
Heat-1D & 0.243 & 0.200 & 0.043 & 0.2704 & 0.2811 & 0.6642 & 1.9922 & 0.394 & 0.431 \\
Heat-2D & 1.000 & 1.000 & 0.000 & 0.0810 & 0.0726 & 0.2463 & 0.7080 & 0.978 & 0.980 \\
Heat-3D & 1.000 & 1.000 & 0.000 & 0.0635 & 0.0577 & 0.1812 & 0.5909 & 0.986 & 0.987 \\
Darcy-2D & 0.688 & 0.662 & 0.025 & 0.1659 & 0.1702 & 0.8937 & 1.0360 & --- & --- \\
Advection--Diffusion-2D & 1.000 & 1.000 & 0.000 & 0.0547 & 0.0505 & 0.9881 & 1.2310 & 0.946 & 0.940 \\
Allen--Cahn-2D & 1.000 & 0.163 & 0.838 & 0.0938 & 0.1220 & 0.2845 & 0.1813 & 0.970 & 0.972 \\
Navier--Stokes-2D & 1.000 & 1.000 & 0.000 & 0.0532 & 0.0289 & 7.0189 & 12.6399 & 0.085 & 0.111 \\
\bottomrule
\end{tabular}
}
\end{table*}

\begin{table}[t]
\centering
\caption{Inference-time per-iteration structure. JCP is a training-time penalty and adds no runtime cost to the D-IPG update. When Armijo acceptance is used, D-IPG additionally evaluates a directional derivative through \(f_W\), and this cost is included in the reported wall-clock timings.}
\label{tab:per_iteration_operations}
\small
\setlength{\tabcolsep}{4.5pt}
\begin{tabular}{lll}
\toprule
Method & Core per-iteration operation & Main cost driver \\
\midrule
GD & \(J_f(x_t)^\top r_t\) & one reverse-mode gradient \\
D-IPG & \(f_W(x_t)\), then \(g_V(f_W(x_t)-\alpha_t r_t)\) & network evaluations \(+\) Armijo directional derivative \\
GN & solve \((J_f^\top J_f)\Delta x=-J_f^\top r_t\) & Jacobian-based linear solve \\
LM & solve \((J_f^\top J_f+\lambda I)\Delta x=-J_f^\top r_t\) & damped Jacobian-based linear solve \\
L-BFGS & gradient \(+\) two-loop recursion & line-search-dependent quasi-Newton step \\
\bottomrule
\end{tabular}
\end{table}

\paragraph{Seed robustness.}
We additionally ran a three-seed robustness study for each benchmark. The main empirical conclusions are stable across seeds. In particular, D-IPG retains perfect or near-perfect success on Heat-2D, Heat-3D, Advection--Diffusion-2D, and Navier--Stokes-2D where applicable, while the qualitative advantage of \(+\)JCP over \(-\)JCP on Allen--Cahn-2D persists under reseeding. On Darcy-2D, the same pattern remains visible with moderate variability, and Heat-1D continues to appear as a genuine failure mode rather than an artifact of a single run. These seed studies support the robustness of the main-text conclusions: the strong performance on Heat-3D, Advection--Diffusion-2D, and Allen--Cahn-2D persists under reseeding, while Heat-1D and Darcy-2D remain the principal benchmark-specific limitations.

\begin{table}[t]
\centering
\caption{Compact three-seed robustness summary. Values report mean \(\pm\) standard deviation across seeds. Best baseline is selected by mean success rate within each benchmark.}
\label{tab:seed_robustness_compact}
\small
\setlength{\tabcolsep}{4.5pt}
\begin{tabular}{lcccc}
\toprule
Benchmark & D-IPG(+JCP) SR & D-IPG(-JCP) SR & Best baseline SR & D-IPG(+JCP) time (s) \\
\midrule
Heat-3D & \(1.000 \pm 0.000\) & \(1.000 \pm 0.000\) & \(1.000 \pm 0.000\) (LM) & \(0.029 \pm 0.000\) \\
Darcy-2D & \(0.638 \pm 0.038\) & \(0.596 \pm 0.058\) & \(0.533 \pm 0.056\) (LM) & \(0.019 \pm 0.016\) \\
Adv.-Diff.-2D & \(0.992 \pm 0.014\) & \(0.996 \pm 0.007\) & \(0.629 \pm 0.101\) (LM) & \(0.168 \pm 0.007\) \\
Allen--Cahn-2D & \(0.950 \pm 0.087\) & \(0.392 \pm 0.255\) & \(0.758 \pm 0.295\) (L-BFGS) & \(0.149 \pm 0.002\) \\
\bottomrule
\end{tabular}
\end{table}

\section{Implementation Details}
\label{app:implementation}

\subsection{Armijo acceptance with amortized proposals}
\label{app:armijo_amortized}

In standard line-search analysis, sufficiently small step sizes are accepted when the search direction is descent and the proposed step vanishes as the step size tends to zero. D-IPG differs slightly because the proposal is generated by the learned reverse map,
\[
x^{\mathrm{prop}}(\alpha)
=
g_V(f_W(x_t)-\alpha r_t).
\]
As \(\alpha\to 0\), this proposal approaches \(g_V(f_W(x_t))\), not necessarily \(x_t\). Therefore, if the zero-order reconstruction error \(g_V(f_W(x_t))-x_t\) is large, the induced step need not vanish and Armijo backtracking may reject all tested step sizes. In our implementation, backtracking is therefore used as a practical acceptance safeguard: when no tested proposal gives sufficient decrease, the solver terminates rather than accepting an unreliable update. This behavior is included in the reported wall-clock timings and accepted-step statistics.

\subsection{Experimental setup and training details}
\label{app:experimental_setup}

All experiments were implemented in Python with PyTorch. Random seeds were synchronized across \texttt{random}, NumPy, and PyTorch, with \texttt{torch.cuda.manual\_seed\_all} additionally used when CUDA was available. Unless otherwise noted, each benchmark used a learned forward map \(f_W\), a learned reverse map \(g_V\), and the D-IPG solver defined in Section~\ref{sec:method}. The \(+\mathrm{JCP}\) and \(-\mathrm{JCP}\) variants shared the same forward model and differed only in the final reverse-map training stage, where the Jacobian Composition Penalty was either included or omitted.

\paragraph{Compute resources.}
All benchmarks were run on a single NVIDIA T4 GPU with CUDA support in Google Colab, together with standard multi-core CPU resources. For the main-paper benchmarks, a typical primary run required approximately \(1\)--\(1.5\) hours, while the corresponding seed-based robustness study required approximately \(3\)--\(4\) hours. Across all reported experiments, the total computational budget was on the order of tens of GPU-hours, excluding minor postprocessing costs such as CSV aggregation and figure generation.

\paragraph{Common training protocol.}
Most benchmarks followed a three-stage procedure. First, the forward surrogate \(f_W\) was trained to map latent states to observations. Second, \(f_W\) was frozen and the reverse map \(g_V\) was pretrained using reconstruction and cycle-consistency objectives. Third, two reverse models were initialized from the same pretrained state and trained separately with and without JCP, yielding the \(+\mathrm{JCP}\) and \(-\mathrm{JCP}\) variants used in evaluation. Across the CNN-based PDE benchmarks, optimization used Adam with cosine-annealing learning-rate schedules, and gradients were clipped to a maximum norm of \(5.0\). Model selection combined reconstruction quality with RJCP, and the final reverse stage used a reconstruction guard to prevent improvements in composition quality from being accepted at the cost of substantial degradation in state recovery. The benchmark-specific epoch counts, learning rates, and loss weights are listed in Table~\ref{tab:hyperparameters}; additional solver-specific settings are provided in the released benchmark scripts.

\begin{table}[!t]
\centering
\caption{Per-benchmark training hyperparameters. Stages S1--S3 denote forward-surrogate training, reverse-map pretraining, and reverse-map fine-tuning with JCP, respectively. All runs use Adam with weight decay \(10^{-6}\), \(\lambda_{\mathrm{task}}=1.0\), and \(\lambda_{\mathrm{rec}}=1.0\).}
\label{tab:hyperparameters}
\small
\setlength{\tabcolsep}{3.5pt}
\begin{tabular}{lcccccccc}
\toprule
& \multicolumn{3}{c}{Epochs} & \multicolumn{3}{c}{Learning rate} 
& \multicolumn{2}{c}{Loss weights} \\
\cmidrule(lr){2-4} \cmidrule(lr){5-7} \cmidrule(lr){8-9}
Problem & S1 & S2 & S3 & \(\eta_1\) & \(\eta_2\) & \(\eta_3\) 
& \(\lambda_{\mathrm{cyc}}\) & \(\lambda_{\mathrm{JCP}}\) \\
\midrule
Heat-1D            & 140 & 100 & 120 & \(2{\times}10^{-3}\) & \(2{\times}10^{-3}\) & \(10^{-3}\) & 0.25 & ---\(^{\dagger}\) \\
Heat-2D            & 160 & 120 & 140 & \(2{\times}10^{-3}\) & \(2{\times}10^{-3}\) & \(10^{-3}\) & 0.15 & 0.50 \\
Heat-3D            & 140 & 100 & 120 & \(2{\times}10^{-3}\) & \(2{\times}10^{-3}\) & \(10^{-3}\) & 0.15 & 0.35 \\
Darcy-2D           & 140 &  60 &  80 & \(2{\times}10^{-3}\) & \(2{\times}10^{-3}\) & \(10^{-3}\) & 0.20 & 0.20 \\
Allen--Cahn-2D     & 120 &  80 &  40 & \(2{\times}10^{-3}\) & \(10^{-3}\) & \(5{\times}10^{-4}\) & 0.05 & 0.01 \\
Adv.--Diff.-2D     & 120 &  80 &  40 & \(2{\times}10^{-3}\) & \(10^{-3}\) & \(5{\times}10^{-4}\) & 0.05 & 0.01 \\
Navier--Stokes-2D  & 120 &  80 &  15 & \(2{\times}10^{-3}\) & \(10^{-3}\) & \(5{\times}10^{-4}\) & 0.05 & \(10^{-3}\) \\
\bottomrule
\multicolumn{9}{l}{\footnotesize \(^{\dagger}\) Heat-1D uses MLP-specific stabilization losses \(\lambda_{\mathrm{bias}}=5{\times}10^{-4}\) and \(\lambda_{\mathrm{comp}}=10^{-3}\),}\\
\multicolumn{9}{l}{\footnotesize rather than the probe-JCP weight used in the CNN-based benchmarks.}
\end{tabular}
\end{table}

\paragraph{Inference-time solvers.}
For a target observation \(y^\star\) and current iterate \(x_t\), D-IPG used the update
\[
y_{t+1}^{\mathrm{prop}} = f_W(x_t)-\alpha_t\bigl(f_W(x_t)-y^\star\bigr),
\qquad
x_{t+1}^{\mathrm{prop}} = g_V\!\left(y_{t+1}^{\mathrm{prop}}\right),
\]
followed by Armijo backtracking on the surrogate objective
\[
\Phi(x)=\frac{1}{2}\,\mathrm{mean}\!\left((f_W(x)-y^\star)^2\right).
\]
Accepted proposals were relaxed and projected to box constraints exactly as in Algorithm~\ref{alg:dipg}. When Armijo acceptance is used, D-IPG additionally evaluates the directional derivative \(\nabla\Phi(x_t)^\top p_t\), computed by automatic differentiation through \(f_W\); this cost is included in the reported wall-clock timings. Unless overridden by benchmark-specific settings, the inverse-problem experiments used Armijo constant \(c=10^{-4}\), relaxation parameter \(\rho=0.4\), and at most \(8\) backtracking steps. The main baselines were Gradient Descent, Gauss--Newton, and Levenberg--Marquardt, with L-BFGS included in several two-dimensional PDE benchmarks. In the PDE experiments, Gauss--Newton and Levenberg--Marquardt linear systems were solved using conjugate gradient applied through normal-equation operators, while L-BFGS used the standard PyTorch implementation.

\paragraph{Training-time amortization.}
The inference-time speedups in the main paper exclude the one-time cost of training the Deceptron module. To quantify the amortization regime, we measured
\[
T_{\mathrm{train}}
=
T_{\mathrm{stage1}}+T_{\mathrm{stage2}}+T_{\mathrm{stage3,+JCP}},
\]
where the three terms correspond to training the forward surrogate, pretraining the reverse map, and then continuing reverse-map training with the JCP objective. Measurements were performed on a single NVIDIA T4 GPU in Google Colab without High-RAM mode. Seed sweeps, ablations, figure generation, CSV postprocessing, and trace export were excluded. For each baseline with average per-instance solve cost \(C_{\mathrm{base}}\), and D-IPG solve cost \(C_{\mathrm{DIPG}}\), the amortized break-even point is
\[
N_{\mathrm{break}}
=
\frac{T_{\mathrm{train}}}{C_{\mathrm{base}}-C_{\mathrm{DIPG}}},
\qquad
C_{\mathrm{base}}>C_{\mathrm{DIPG}}.
\]
Thus \(N_{\mathrm{break}}\) estimates the number of inverse instances after which the one-time training cost is recovered relative to that baseline. Including \(T_{\mathrm{stage1}}\) is conservative for D-IPG, since the same trained forward surrogate is also used by the classical baselines evaluated on the surrogate objective; if forward training is treated as a shared prerequisite, the reported break-even points decrease. Table~\ref{tab:amortization_break_even} reports the corresponding wall-clock measurements.

\begin{table}[!t]
\centering
\caption{Amortized training-cost break-even analysis on the three main-text benchmarks.}
\label{tab:amortization_break_even}
\small
\setlength{\tabcolsep}{5pt}
\begin{tabular}{llrrrr}
\toprule
Benchmark & Baseline &
\(T_{\mathrm{train}}\) &
\(C_{\mathrm{DIPG}}\) &
\(C_{\mathrm{base}}\) &
\(N_{\mathrm{break}}\) \\
\midrule
Heat-3D & GN
& 1238.28 & 0.03395 & 1.01181 & 1266.33 \\
Heat-3D & LM
& 1238.28 & 0.03395 & 0.80940 & 1596.86 \\
\midrule
Adv.-Diff.-2D & GN
& 1596.15 & 0.13526 & 3.05891 & 545.95 \\
Adv.-Diff.-2D & LM
& 1596.15 & 0.13526 & 3.03719 & 550.03 \\
Adv.-Diff.-2D & L-BFGS
& 1596.15 & 0.13526 & 1.07835 & 1692.47 \\
\midrule
Allen--Cahn-2D & GN
& 1592.89 & 0.14720 & 2.86459 & 586.19 \\
Allen--Cahn-2D & LM
& 1592.89 & 0.14720 & 11.62872 & 138.74 \\
Allen--Cahn-2D & L-BFGS
& 1592.89 & 0.14720 & 0.41055 & 6048.76 \\
\bottomrule
\end{tabular}
\end{table}

Table~\ref{tab:amortization_break_even} shows that the upfront training cost is recovered after a few hundred to a few thousand inverse solves on the representative benchmarks, depending on the baseline. The break-even point is smallest when the baseline repeatedly incurs expensive Jacobian-based solves, as in LM on Allen--Cahn-2D, and largest against relatively cheap quasi-Newton solves such as L-BFGS. These results support the intended use of D-IPG as an amortized inverse solver: the method is most advantageous when many inverse instances are solved for the same forward family, while one-off inverse problems may not benefit from the training-time investment.

\FloatBarrier

\paragraph{Architectures, benchmark families, and logged outputs.}
Heat-1D used an MLP-style Deceptron, with a shallow forward map and a slightly more expressive reverse map. The spatial benchmarks used lightweight CNN-based Deceptron models, with both \(f_W\) and \(g_V\) implemented as shallow residual convolutional networks matched to the dimensionality of the task. The benchmark suite consisted of \(1\)-, \(2\)-, and \(3\)-dimensional heat-equation initial-condition recovery, together with harder two-dimensional PDE inverse problems. Heat-1D used a nonlinear observation operator built on top of the one-dimensional heat semigroup. Heat-2D and Heat-3D used nonlinear diffusive forward processes with blur and pointwise observation distortions. Darcy-2D estimated permeability fields from pressure observations generated by an elliptic PDE solver. Advection--Diffusion-2D used a Fourier-domain transport--diffusion propagator with degraded observations. Allen--Cahn-2D used a spectral semi-implicit phase-field solver. Navier--Stokes-2D used a pseudo-spectral incompressible vorticity solver and future-snapshot observations. Across all benchmarks, datasets were generated synthetically from procedurally sampled latent fields, split into training, validation, and test sets, normalized using statistics computed from the training observations, and constrained at inference time whenever appropriate for the underlying state variable. We report success rate, iteration counts, state RMSE, final residual, accepted-step fraction, and wall-clock time; for D-IPG, we additionally record RJCP and step-level diagnostics such as accepted step size, step norm, residual decrease, and cosine similarity with the negative gradient.   

\end{document}